\documentclass[10pt,twocolumn,letterpaper]{article}

\usepackage[pagenumbers]{cvpr} 

\usepackage{pifont}

\definecolor{cvprblue}{rgb}{0.21,0.49,0.74}
\usepackage[pagebackref,breaklinks,colorlinks,allcolors=cvprblue]{hyperref}

\usepackage{url}
\usepackage{amsmath,amssymb,amsthm,mathtools}
\usepackage{bm}
\usepackage[table]{xcolor}
\usepackage{enumitem}

\usepackage{subcaption}
\usepackage{adjustbox}
\usepackage{amsfonts}
\theoremstyle{plain}
\newtheorem{theorem}{Theorem}
\newtheorem{lemma}[theorem]{Lemma}

\newtheorem{corollary}[theorem]{Corollary}
\theoremstyle{definition}
\newtheorem{definition}{Definition}
\theoremstyle{remark}
\newtheorem{remark}[theorem]{Remark}
\usepackage{multirow} 
\usepackage{makecell}
\definecolor{cvprblue}{rgb}{0.21,0.49,0.74}

\def\paperID{18545} 
\def\confName{CVPR}
\def\confYear{2026}

\title{Learning Straight Flows: Variational Flow Matching for Efficient Generation}

\author{
\textbf{Chenrui Ma}\textsuperscript{1} \quad
\textbf{Xi Xiao}\textsuperscript{2} \quad
\textbf{Tianyang Wang}\textsuperscript{2} \quad
\textbf{Xiao Wang}\textsuperscript{3} \quad
\textbf{Yanning Shen}\textsuperscript{1}\thanks{Corresponding author: yannings@uci.edu.}
\vspace{1mm}
\\
\textsuperscript{1}University of California, Irvine, Irvine, CA 92697, USA
\\
\textsuperscript{2}University of Alabama at Birmingham, Birmingham, AL 35294, USA
\\
\textsuperscript{3}Oak Ridge National Laboratory, Oak Ridge, TN 37830, USA
}

\begin{document}
\maketitle

\begin{abstract}
Flow Matching has limited ability in achieving one-step generation due to its reliance on learned curved trajectories. Previous studies have attempted to address this limitation by either modifying the coupling distribution to prevent interpolant intersections or introducing consistency and mean-velocity modeling to promote straight trajectory learning. However, these approaches often suffer from discrete approximation errors, training instability, and convergence difficulties. To tackle these issues, in the present work, we propose \textbf{S}traight \textbf{V}ariational \textbf{F}low \textbf{M}atching (\textbf{S-VFM}), which integrates a variational latent code representing the ``generation overview'' into the Flow Matching framework. \textbf{S-VFM} explicitly enforces trajectory straightness, ideally producing linear generation paths. The proposed method achieves competitive performance across three challenge benchmarks and demonstrates advantages in both training and inference efficiency compared with existing methods.
\end{abstract}

\section{Introduction}
\label{sec:intro}

The goal of generative modeling is to transform a simple prior distribution into a complex data distribution. Flow Matching~\cite{lipman2023flow, albergo2023building, esser2024scaling} pursues this objective by learning a velocity field whose associated ODE transports samples from the source distribution to the target distribution~\cite{lipman2024flow}. Although Flow Matching adopts linear interpolants as training trajectories, the generative paths it learns are typically \emph{curved}~\cite{geng2025mean, song2023consistency}. This curvature arises because the commonly used independent coupling between the source and target distribution results in many intersecting linear interpolants, forcing the learned velocity field to take an average of incompatible directions~\cite{frans2025one, kim2024consistency}. As a result, Flow Matching often requires many ODE integration steps to maintain sample quality, limiting its efficiency~\cite{hoogeboom2023simple}.

Various methods attempt to alleviate this curvature problem. Some modify coupling strategies~\cite{pooladian2023multisample, ma2025stochastic, albergo2024stochastic, zhang2025hierarchical} or learn improved coupling distributions, aiming to reduce interpolant intersections~\cite{silvestri2025vct, tong2023improving, klein2023equivariant}. Others, such as Rectified Flows~\cite{liu2023flow, zhang2025towards, lee2024improving, esser2024scaling, geng2023onestep, salimans2022progressive, wang2024rectified}, iteratively refine the model in multiple distillation stages so that its trajectories gradually approach straight, optimal-transport paths~\cite{villani2008optimal, albergo2023stochastic}. Meanwhile, Consistency Models~\cite{song2024improved, kim2024consistency, song2023consistency, sabour2025align} and Mean Velocity Models~\cite{frans2025one, geng2025mean, zhang2025alphaflow} enforce temporal consistency in the learned dynamics, encouraging trajectories at different time steps to agree. Although these approaches embody different design principles—better couplings, iterative refinement, or self-consistency—they share a common motivation: to learn straight generative trajectories. However, they all inherit a fundamental limitation. \textbf{Because the independent coupling inherently produces intersecting interpolants, any method that attempts to learn straight trajectories directly from this training structure faces an intrinsic contradiction}~\cite{guo2025variational, liu2023flow}. This misalignment often manifests as training instability, discretization artifacts, or difficulty in converging toward truly straight paths.

In this work, we address this contradiction from a new direction. Rather than enforcing straight trajectories under the independent-coupling structure that inherently causes interpolant intersections, we equip the model with the ability to \emph{disambiguate} these intersections. To this end, we introduce the \textbf{Straight Variational Flow Matching} framework, which \ding{182} leverages a variational latent code~\cite{kingma2013auto, burgess2018understanding, anonymous2025cadvae} to provide a global ``generation overview'' of each source--target pair, enabling the velocity field to choose the correct direction even when interpolants intersecting. Furthermore, we promote \ding{183} straight trajectory learning by penalizing the time variation of the velocity field along the generative path. Together, these components yield naturally straighter trajectories that require far fewer ODE steps to simulate while maintaining high generation quality~\cite{dao2025self, feng2024relational}.

Extensive experiments demonstrate that the proposed method achieves state-of-the-art performance in generation quality, sampling efficiency, and training stability across multiple benchmark datasets. Our theoretical analysis critically clarifies why existing methods struggle to learn straight trajectories under independent couplings, and shows how our formulation resolves this challenge in a principled and effective manner.

In summary, our contributions are as follows.
\begin{itemize}
    \item We identify and prove the fundamental contradiction between straight-trajectory learning and the independent coupling structure used in Flow Matching and related approaches, presenting a theoretical motivation for the proposed method.
    \item We introduce a novel framework that overcomes this contradiction by incorporating a global latent code for disambiguating intersecting trajectories and a straightness objective that encourages time-invariant velocity fields.
    \item Superior experimental results across three challenge benchmarks confirm our theoretical claims and demonstrate substantial improvements in generation fidelity, efficiency, and stability.
\end{itemize}

\section{Related Works}
\label{sec:rela}

\subsection{Rectified Flow and Distillation}

To address the issue of curved generation trajectories, Rectified Flow~\cite{liu2023flow, zhang2025towards, lee2024improving, esser2024scaling, geng2023onestep, salimans2022progressive, wang2024rectified} reformulates the flow-matching learning process into multiple iterative stages. Specifically, it first trains a basic flow-matching model, then uses the data pairs generated by this model to train a new one~\cite{zhang2025hierarchical}. This process is repeated iteratively. Essentially, each subsequent model refines the quality of the training data pairs—reducing trajectory intersections and curvature—by leveraging the outputs of the previous model~\cite{liu2023flow}. This pipeline is also known as distillation~\cite{salimans2022progressive}.
Through this progressive refinement, the model gradually approaches an optimal transport~\cite{kornilov2024optimal} solution, characterized by straight, non-intersecting trajectories within the generation ODE.

Despite its success in reducing curvature and improving trajectory consistency, Rectified Flows and distillation methods suffer from inefficiency: obtaining the final model requires multiple rounds of training, which are both time- and resource-intensive~\cite{wan2024cad, mei2024codi}. 
Moreover, this iterative process may introduce error accumulation, since each subsequent model is trained on data generated by its predecessor rather than the true data distribution~\cite{yin2024improved}. Consequently, the final distilled model often struggles to surpass the generation fidelity of the initial model trained directly on real data~\cite{esser2024scaling, cai2025diffusion}.
Plus, these multi-stage and tightly-scheduled procedures suffer from a need to specify when to end training and begin distillation~\cite{salimans2022progressive, yin2024one}. In contrast, end-to-end methods are efficient, accurate, and can be trained indefinitely to continually improve.

\subsection{Consistency Model and Mean Velocity Model}
Consistency Models represent an important step in straight flow learning, functioning as standalone generative models that eliminate the need for explicit distillation procedures~\cite{song2023consistency, geng2025consistency}. These models enforce consistency constraints between network outputs at different time steps, ensuring that trajectories converge to identical endpoints~\cite{kim2024consistency, lu2025simplifying}. Various formulations and training schemes have been proposed to enhance their stability and sample quality~\cite{song2024improved}.

Mean Velocity Models focus on defining and leveraging flow-based properties between temporal states. For instance, Shortcut Models~\cite{frans2025one} integrate self-consistency losses with flow matching to capture correlations among flows across discrete time intervals, whereas Inductive Moment Matching~\cite{zhou2025inductive} extends this principle to stochastic interpolants, enforcing consistency of their moments over time. Building on this, Meanflows~\cite{geng2025mean} introduce a new ground-truth field representing the mean velocity and model it by the identity relationship between mean velocity and instantaneous velocity, while Flow Maps~\cite{boffi2024flow} is introduced as integrals of flow fields between two time steps, providing a displacement-based characterization complementary to the average-velocity perspective. Together, these methods aim to bridge the gap between local flow dynamics and global consistency in generative processes. Recent studies have sought to unify Consistency Models and Mean Velocity Models, demonstrating that these two classes of methods are mathematically equivalent, differing primarily in their model parameterizations and training procedures~\cite{hu2025cmt, you2025modular}. Although both Consistency Models and Mean Velocity Models move toward a fully end-to-end generative paradigm, their heavy reliance on bootstrapping necessitates carefully scheduled training processes~\cite{frans2025one, kim2024consistency}. Moreover, these approaches often suffer from discrete approximation errors~\cite{geng2025consistency, song2023consistency}, training instability~\cite{song2023consistency}, and convergence challenges~\cite{geng2025mean, you2025modular, ji2025cibr, zhang2025alphaflow}.
In contrast, the proposed method builds upon the Variational Flow Matching framework~\cite{guo2025variational} with a straightness objective, representing a distinct direction from the methods discussed above.

\section{Preliminary}
\label{sec:preliminary}
\noindent\textbf{Notation.}
For a random process \(X=\{X_t\}_{t\in[0,1]}\), write \(\mathrm{Law}(X_t)\) for its marginal law at time \(t\),
and \(\mathbb{E}[\cdot]\) for expectation.
\subsection{Flow Matching}
Let \((X_0,X_1)\) be any coupling on \(\mathbb{R}^d\) with joint density \(\rho(x_0,x_1) = \rho_0(x_0)\,\rho_1(x_1)\), where $\rho_0(x_0)$ is source distribution (e.g Gaussian) and $\rho_1(x_1)$ is target data distribution (e.g image samples).
Define the linear interpolation and \emph{conditional velocity}:
\begin{equation}
\label{eq:linear_interpolation}
\begin{aligned}
X_t=(1-t)X_0+tX_1,\quad \Delta^X=X_1-X_0,\quad t\in[0,1].
\end{aligned}
\end{equation}
Flow Matching models the generation process as an ordinary differential equation (ODE):
\begin{equation}
\label{eq:ODE}
\begin{aligned}
\frac{d}{dt} X_t = v^X(X_t,t) \quad \textit{for} \,\,\,\, t\in[0,1].
\end{aligned}
\end{equation}
The associated \emph{marginal velocity} is
\begin{equation}
\label{eq:marginal_cond_vel}
v^X(x,t) \;=\; \mathbb{E}\big[\Delta^X \,\big|\, X_t=x\big],
\end{equation}
The Flow Matching model parameterized by $\theta$ is trained by minimizing the Conditional Flow Matching loss:
\begin{equation}
\mathcal L_{\mathrm{FM}}(\theta)
\;=\;
\mathbb{E}\!\left[\big\|v_\theta(X_t,t) - \Delta^X \big\|^2 \right],
\end{equation}
instead of the Marginal Flow Matching loss, which is intractable~\cite{lipman2023flow, lipman2024flow}.
Although the linear interpolation path is defined in Eq~\eqref{eq:linear_interpolation}, the learned generating ODE trajectories $X$ are curves, since the \emph{marginal velocity} in Eq~\eqref{eq:marginal_cond_vel} equals the conditional expectation of \emph{conditional velocity} $\Delta^X$.
These trajectories $X$ are rectifiable as defined below~\cite{liu2023flow}.
\begin{definition}[Rectifiability of \(X\)] \label{def:rectifiability}
We say that \(X\) is \emph{rectifiable} if \(v^X(\cdot,t)\) is locally bounded for each \(t\) and the \emph{continuity equation}~\cite{lipman2023flow, lipman2024flow}:
\[
\partial_t \pi_t + \nabla\!\cdot\!\big(v^X(\cdot,t)\,\pi_t\big) \;=\;0,\qquad \pi_{t=0}=\mathrm{Law}(X_0),
\]
admits a unique solution \(\{\pi_t\}_{t\in[0,1]}\).
Equivalently, the ordinary differential equation \(\dot X_t=v^X(X_t,t)\) admits a unique flow of characteristics.
\end{definition}

\subsection{Interpolants Intersection}
\begin{definition}[Non-intersection functional]
For any coupling \((X_0,X_1)\) with linear interpolant \(X_t\), define
\begin{equation}
V\big((X_0,X_1)\big) \;=\; \int_0^1 \mathbb{E}\!\left[\big\|\,\Delta^X - \mathbb{E}[\Delta^X \mid X_t]\,\big\|^2\right]\,dt.
\end{equation}
\end{definition}
\begin{lemma}[Non-intersection $\equiv$ zero conditional variance]
\label{lem:non_intersection_equiv}
See the proof in the Appendix.
For a coupling \((X_0,X_1)\) with linear interpolation \(X_t\), the following are equivalent:
\begin{enumerate}
\item For two independent identically distributed couplings \((X_0,X_1)\) and \((X_0',X_1')\),
\begin{align*}
\exists t\in(0,1): &(1-t)X_0+tX_1=(1-t)X_0'+tX_1' \\
&\mathbb{P}\!\left[(X_0,X_1)\neq(X_0',X_1')\right]=0.
\end{align*}
\item \(V\big((X_0,X_1)\big)=0\); equivalently \(\Delta^X=\mathbb{E}[\Delta^X\mid X_t]\) for \(t\in(0,1)\).
\end{enumerate}
\end{lemma}
The curved ODE trajectories $X$ learned by Flow Matching arise from the intersection of interpolants, which is undesirable, as it leads to $V\big((X_0, X_1)\big) \neq 0$. Consequently, multiple sampling steps are required to simulate the ODE trajectory with a tolerable discretization error:
\begin{equation}
\label{eq:ODE_simulation}
X_{t_{i+1}} = X_{t_i} + (t_{i+1} - t_i) \, v_\theta(X_{t_i}, t_i)
\end{equation}
rather than a few-step or even a one-step generation.

\section{Method}
\label{sec:method}

\subsection{Straight Interpolants}
Ideally, the generation trajectories are \emph{straight interpolants}, whose specific definition compatible with Flow Matching trajectories $X$ is shown below.
\begin{definition}[Straight interpolation compatible with \(X\)]
\label{def:Z}
A process \(Z=\{Z_t\}_{t\in[0,1]}\) on the same probability space as \(X\) is called a
\emph{straight interpolants} compatible with Flow Matching trajectories \(X\) if the following hold:
\begin{itemize} [leftmargin=2em]
\item[(Z1)] \textbf{Linear paths.} There exist random endpoints \((Z_0,Z_1)\) with \(Z_t=(1-t)Z_0+tZ_1\) and \(\Delta^Z=Z_1-Z_0\).
\item[(Z2)] \textbf{Non-intersection.} \(V\big((Z_0,Z_1)\big)=0\). Equivalently, \(\Delta^Z=\mathbb{E}[\Delta^Z\mid Z_t]\) for \(t\in[0,1]\).
\item[(Z3)] \textbf{Velocity-field matching.} With \(\mu_t=\mathrm{Law}(Z_t)\), for \(x \sim \mu_t\) and \(t\in[0,1]\).
\[
v^Z(x,t)=\mathbb{E}[\Delta^Z\mid Z_t=x]=v^X(x,t).
\]
\item[(Z4)] \textbf{Initialization.} \(Z_0=X_0\).
\end{itemize}
\end{definition}
The \emph{straight interpolants} $Z$ satisfy several key properties as shown below.
\begin{theorem}[Marginal preservation]
\label{thm:marginal_preservation}
See the proof in the Appendix.
Assume \(X\) is rectifiable and \(Z\) satisfies Definition~\ref{def:Z}.
Then \(\mathrm{Law}(Z_t)=\mathrm{Law}(X_t)\) for all \(t\in[0,1]\).
\end{theorem}
\begin{theorem}[Convex transport-cost reduction]
\label{thm:convex_cost}
See the proof in the Appendix.
Under the assumptions of Theorem~\ref{thm:marginal_preservation}, for any convex \(c:\mathbb{R}^d\to\mathbb{R}\),
\[
\mathbb{E}\!\left[c(Z_1-Z_0)\right] \;\le\; \mathbb{E}\!\left[c(X_1-X_0)\right].
\]
If \(c\) is strictly convex, equality holds if and only if \(V\big((X_0,X_1)\big)=0\).
\end{theorem}
Thus, we further conclude the equivalent traits of \emph{straight interpolants} for Flow Matching trajectories.
\begin{theorem}[Equivalent characterizations of straight interpolants]
\label{thm:equivalences}
See the proof in the Appendix.
Assume \(X\) is rectifiable and \(Z\) satisfies Definition~\ref{def:Z}.
The following are equivalent: when they hold, we say $X$ yielded from couplings \((X_0,X_1)\) are \emph{straight interpolants}, which coincide with $Z$.
\begin{enumerate} [leftmargin=2em]
\item[(i)] \textbf{Tight convex cost.} There exists a strictly convex \(c\) with
\(\mathbb{E}[c(Z_1-Z_0)]=\mathbb{E}[c(X_1-X_0)]\).
\item[(ii)] \textbf{Endpoint coincidence.} \((Z_0,Z_1)=(X_0,X_1)\).
\item[(iii)] \textbf{Pathwise equality.} \(Z_t=X_t\) for all \(t\in[0,1]\).
\item[(iv)] \textbf{Non-intersection for \(X\).} \(V\big((X_0,X_1)\big)=0\).
\end{enumerate}
\end{theorem}
Conclusively, for Flow Matching, the objective of learning straight generation trajectories both implies and is ensured by the condition $V\big((X_0, X_1)\big) = 0$. 
However, this objective is unattainable for Flow Matching because the learned \emph{marginal velocity} in Eq.~\eqref{eq:marginal_cond_vel} is the conditional expectation of the \emph{conditional velocity} $\Delta^X$. As a result, $V\big((X_0, X_1)\big) \neq 0$, caused by the intersections of multiple linear interpolations defined in Eq.~\eqref{eq:linear_interpolation} at $X_t$ under the independent coupling $\rho(x_0, x_1) = \rho_0(x_0)\,\rho_1(x_1)$.

\subsection{Vanishing Time Derivative Along \texorpdfstring{$X$}{X}}
To achieve $V\big((X_0,X_1)\big)=0$, we propose to minimize the time derivative of the learned velocity field: $\frac{d}{dt} v$.
\begin{definition}[Time derivative] \label{def:timederivative}
For a differentiable vector field \(v:\mathbb{R}^d\times[0,1]\to\mathbb{R}^d\), the time derivative along its characteristics is
\begin{equation}
\begin{aligned}
\frac{d}{dt} v(x,t) &= D_t v(x,t)\ := \frac{\partial v}{\partial t} \frac{dt}{dt} \ +\ \frac{\partial v}{\partial x} \frac{dx}{dt}   \\
&= \frac{\partial}{\partial t}v(x,t)\ +\ \big(\nabla_x v(x,t)\big)\,v(x,t)
\end{aligned}
\end{equation}
\end{definition}

\begin{theorem}[Straightness $\equiv$ vanishing time derivative along \(X\)]
\label{thm:material_derivative_equivalence}
See the proof in the Appendix.
Assume \(X\) is rectifiable and \(Z\) satisfies Definition~\ref{def:Z}.
Assume moreover that \(v^X\) is continuously differentiable in \((x,t)\). Then $V\big((X_0,X_1)\big)=0$, if and only if 
$D_t v^X(X_t,t)=0$, for \(t\in[0,1]\).
\end{theorem}
Nevertheless, directly minimizing $D_t v^X(X_t, t)$ conflicts with the Flow Matching objective in Eq.~\eqref{eq:marginal_cond_vel}, since the intersections produced by the independent coupling make the marginal velocity fundamentally incompatible with straight-trajectory learning. Existing methods, as reviewed in the Related Works section, ultimately aim for the same condition $V((X_0, X_1)) = 0$, which is equivalent to enforcing $D_t v^X(X_t, t) = 0$. However, because they operate within this inherently contradictory formulation, their effectiveness remains fundamentally limited.


\subsection{Variational Flow Matching for Straightness}



\begin{figure*}[htbp]
  \centering
  \setlength{\tabcolsep}{0pt}
  \begin{tabular}{ccccc}
    \subcaptionbox{\scriptsize Ground Truth\strut}[0.198\linewidth]{%
      \includegraphics[width=\linewidth]{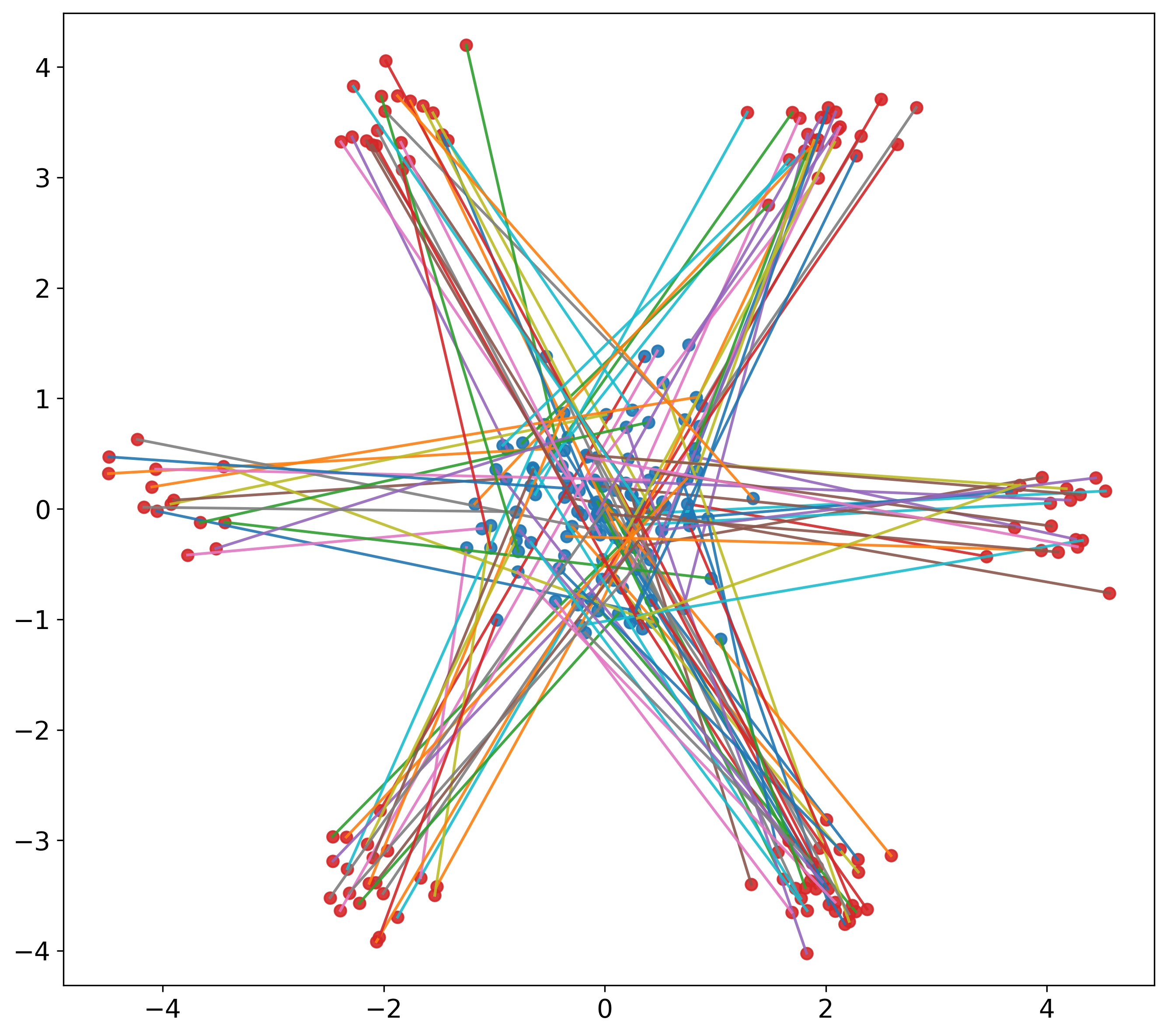}} &
    
    \subcaptionbox{\scriptsize Flow Matching\strut}[0.198\linewidth]{%
      \includegraphics[width=\linewidth]{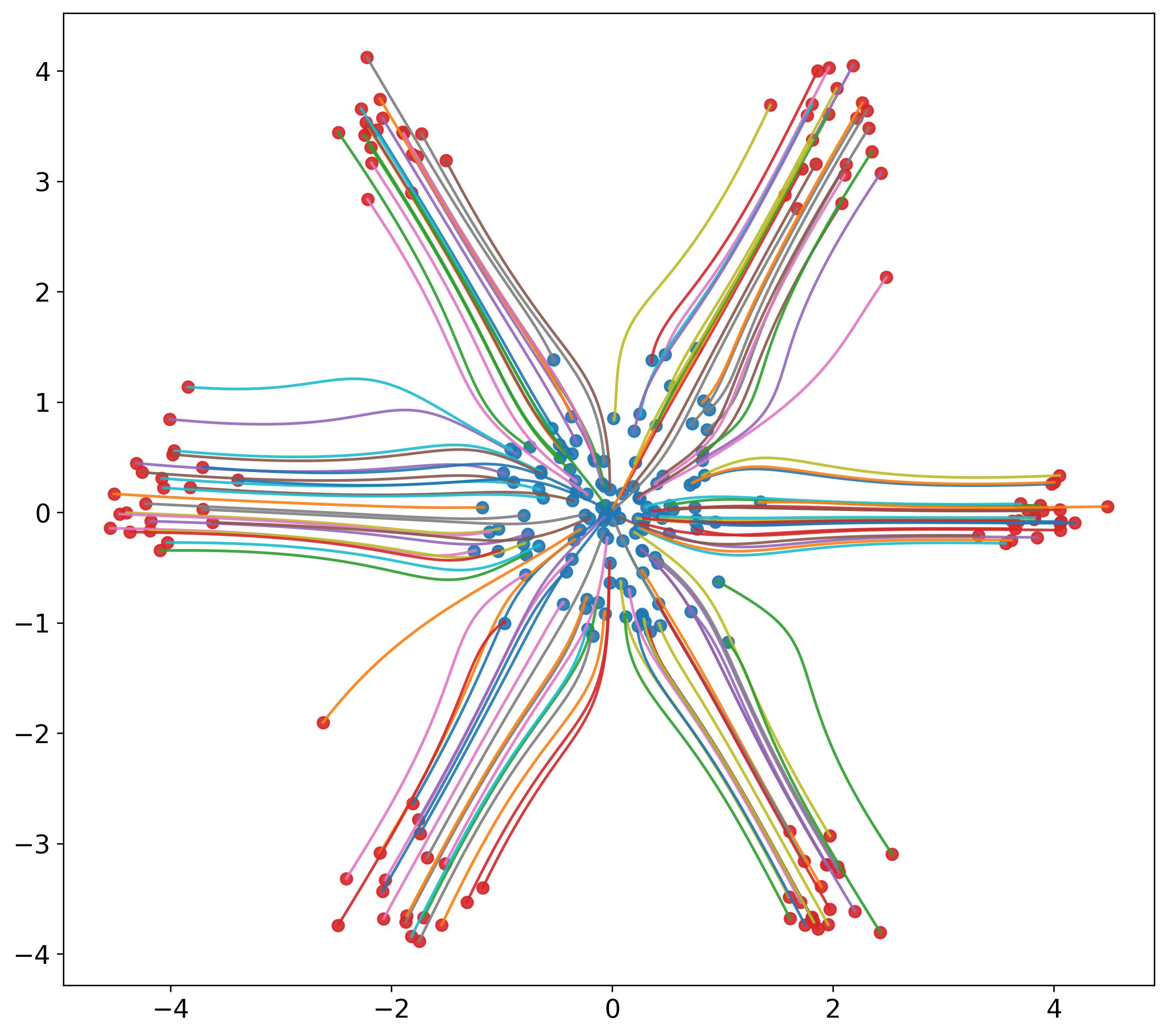}} &
    
    \subcaptionbox{\scriptsize Rectified Flow\strut}[0.198\linewidth]{%
      \includegraphics[width=\linewidth]{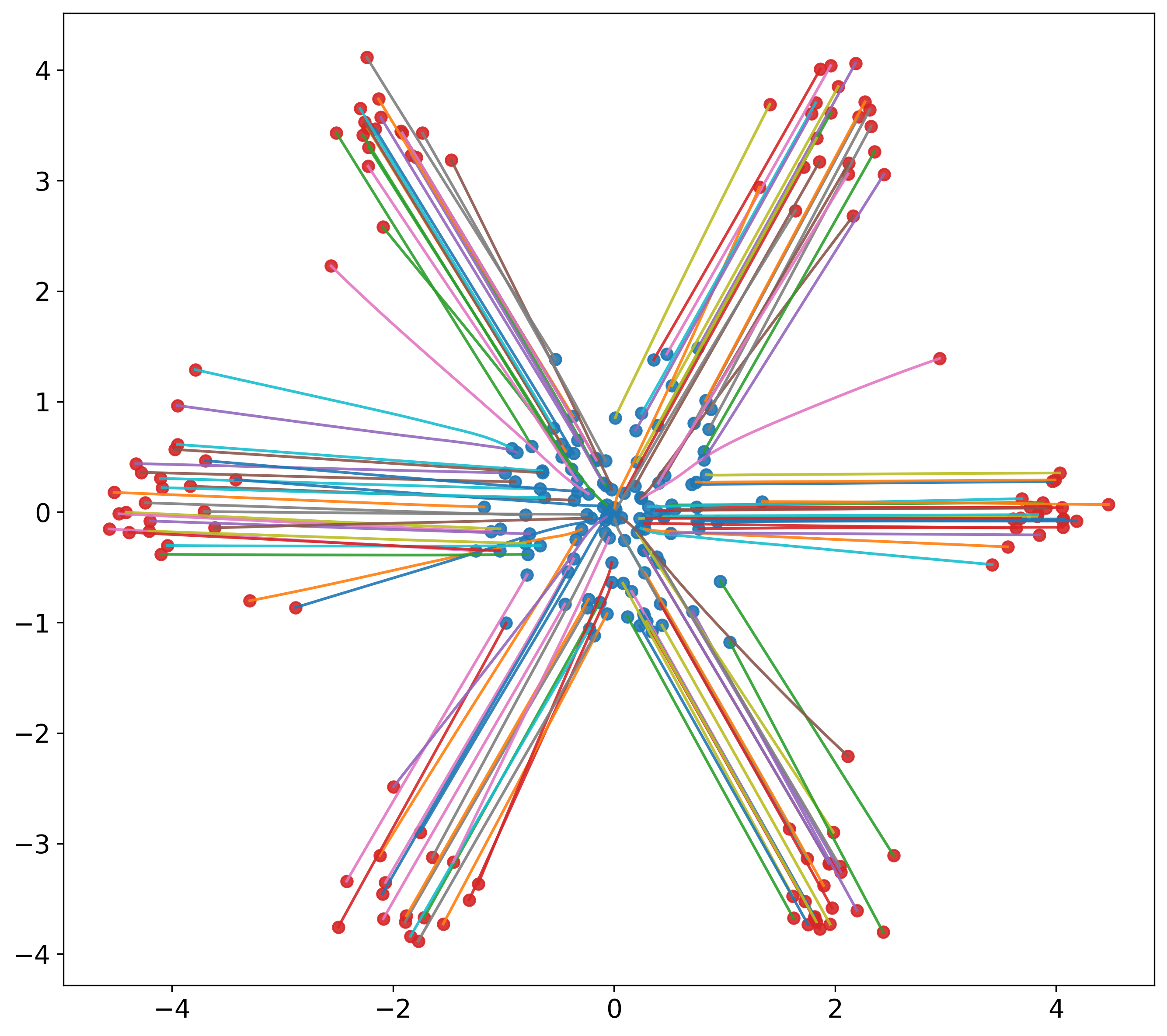}} &
    
    \subcaptionbox{\scriptsize Variational Flow Matching\strut}[0.198\linewidth]{%
      \includegraphics[width=\linewidth]{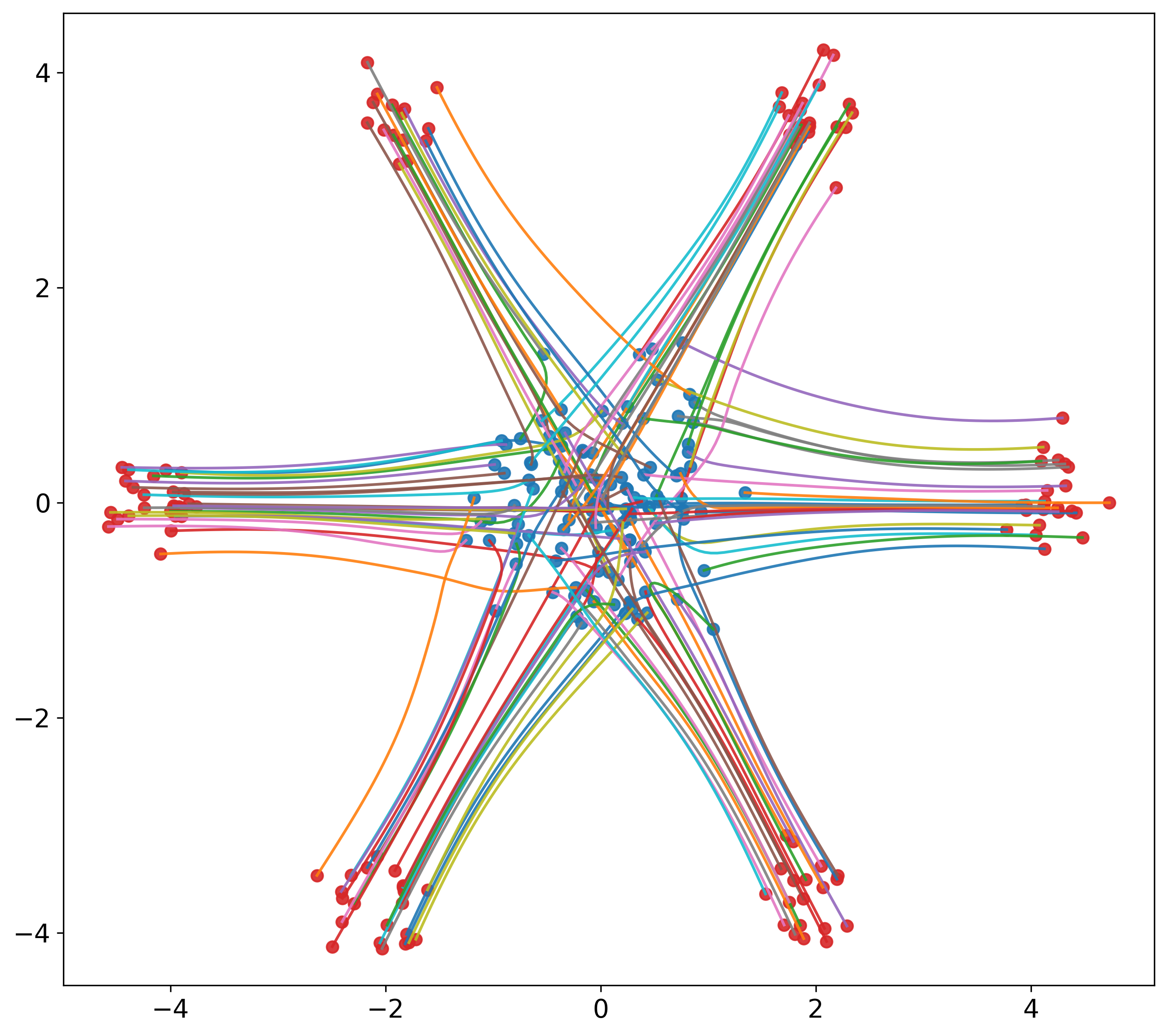}} &
    
    \subcaptionbox{\centering\scriptsize Straight Variational Flow Matching\strut}[0.198\linewidth]{%
      \includegraphics[width=\linewidth]{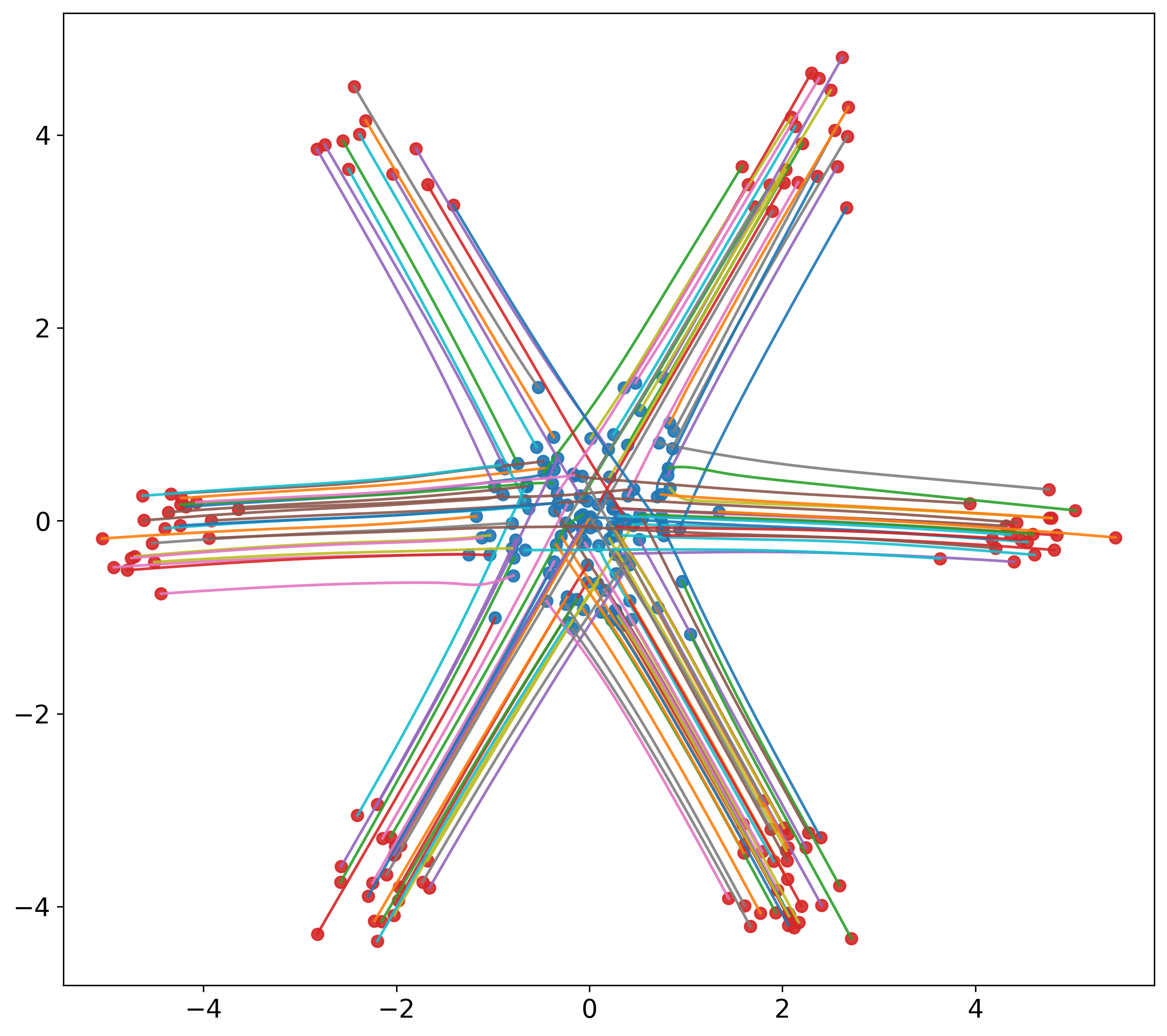}}
  \end{tabular}
  \vspace{-1em}
  \caption{\textbf{Generation Trajectory Visualization in the 2D Synthesized Hexagonal Dataset.}}
  \label{fig:viz_2D_hex}
\end{figure*}

Instead of these existing methods, we introduce Variational Flow Matching with a straightness objective to enforce vanishing time derivative along trajectories, thereby achieving $V\big((X_0, X_1)\big)=0$.

The generation process of Flow Matching is basically a Markov process~\cite{norris1998markov} that predicts the velocity $v^X(X_{t_i}, t_i)$ at time $t_i$ using the current sample $X_{t_i}$ to obtain the next sample $X_{t_{i+1}}$ at time $t_{i+1}$, as shown in Eq.~\eqref{eq:ODE_simulation}. This procedure implies that the model lacks awareness of the ``generation overview'' of the trajectories, which is a key reason why interpolants intersect and the learned ODE trajectories become curved. 
Variational Flow Matching (VFM)~\cite{guo2025variational}, parameterized by $\theta$, incorporates a latent code $z$ obtained via a variational autoencoder (VAE) parameterized by $\phi$:
\begin{equation}
\begin{aligned}
q_\phi(z &\big| X_0,X_1,X_t,t) = \\ &\mathcal{N}(z; \mu_\phi(X_0,X_1,X_t,t), \sigma_\phi(X_0,X_1,X_t,t)),
\end{aligned}
\end{equation}
which involves the ``generation overview'' information \((X_0, X_1)\). To predict the velocity \(v(X_t, t, z)\), we minimize:
\begin{equation}
\label{eq:loss_vfm}
\begin{aligned}
\mathcal L_{\mathrm{VFM}}(\theta, \phi)
\;=\;
\mathbb{E}&\left[\big\|v_\theta(X_t,t,z) - \Delta^X \big\|^2 \right] + \\
&\beta D_{KL}\big(q_\phi(z \big| X_0,X_1,X_t,t) \big| p(z)\big),
\end{aligned}
\end{equation}
where KL divergence term prevent the posterior distribution $q_\phi(z \big| X_0,X_1,X_t,t)$ far from prior $p(z)=\mathcal{N}(z; 0, I)$, and its strength is governed by the hyperparameter $\beta$~\cite{burgess2018understanding}.

We do not seek to achieve non-intersection for $X$ as in Theorem~\ref{thm:equivalences}. Instead, by leveraging the ``generation overview'' information provided by $z$, we aim to obtain a flow model $v_\theta(X_t, t, z)$ that can distinguish ``where to go'' even when interpolants are intersected for $X$. 
Furthermore, to achieve $V\big((X_0, X_1)\big)=0$, we minimize $D_t v_\theta(X_t, t, z)$. Note that, since VFM incorporates the latent code $z$, it can handle trajectory intersections; consequently, $Z$ and $X$ become compatible with each other under independent couplings $\rho(x_0, x_1) = \rho_0(x_0)\,\rho_1(x_1)$ for VFM. This is a key distinction from Flow Matching. See Figure~\ref{fig:viz_2D_hex}.
The straightness objective $D_t v(X_t,t,z)$ extends Definition~\ref{def:timederivative}: 
\begin{equation}
\label{eq:Dt_v}
\begin{aligned}
D_t v &(X_t,t,z) = \partial_{X_t}v \cdot \frac{d X_t}{dt} + \partial_{t}v \cdot \frac{d t}{dt} + \partial_{z}v \cdot \frac{d z}{dt} \\
&= \partial_{X_t}v \cdot v^X(X_t, t, z) + \partial_{t}v \cdot 1 + \partial_{z}v \cdot \frac{d z}{dt}
\end{aligned}
\end{equation}
where $\frac{d X_t}{dt} = v^X(X_t, t, z)$ extending from Eq.~\eqref{eq:ODE}. 
Further, calculate $\frac{d z}{dt}$ in the same way:
\begin{equation}
\label{eq:Dt_z}
\begin{aligned}
\frac{d z}{dt} &= \partial_{X_0}z \cdot \frac{d X_0}{dt} + \partial_{X_1}z \cdot \frac{d X_1}{dt} + \partial_{X_t}z \cdot \frac{d X_t}{dt} + \partial_{t}z \cdot \frac{d t}{dt} \\
&= \partial_{X_0}z \cdot 0 + \partial_{X_1}z \cdot 0 + \partial_{X_t}z \cdot v^X(X_t, t, z) + \partial_{t}z \cdot 1
\end{aligned}
\end{equation}
The time derivative shown in Eq.~\eqref{eq:Dt_v} \eqref{eq:Dt_z} is given by the Jacobian-vector product (JVP) between the Jacobian matrix of each function and the corresponding tangent vector. For code implementation, modern libraries such as PyTorch provide efficient JVP calculation interfaces. Different from previous studies~\cite{geng2025mean} using \texttt{torch.func.jvp}, which perform in forward mode and do not retain the computation graph for backward passes, \texttt{torch.autograd.functional.jvp} instead retains the graph, thus suitable for our method.
In practice, the velocity field $v(X_t,t,z)$ is parameterized as $v_\theta(X_t,t,z)$, and we follow \cite{lipman2023flow} to replace the marginal velocity $v^X(X_t, t, z)$ with the conditional velocity $\Delta^X$.
Combining with Eq.~\eqref{eq:Dt_v} \eqref{eq:Dt_z}, the straightness objective is achieved by minimizing:
\begin{equation}
\label{eq:loss_s}
\begin{aligned}
\mathcal L_{\mathrm{S}}(\theta, \phi) \!=\! \left[\big\|\partial_{X_t}v \!\cdot\! \Delta^X \!+ \partial_{t}v + \partial_{z}v \!\cdot\! (\partial_{X_t}z \!\cdot\! \Delta^X \!+ \partial_{t}z)\big\|^2 \right] 
\end{aligned}
\end{equation}

Thus, the total loss function is formed by Eq.~\eqref{eq:loss_vfm} \eqref{eq:loss_s}:
\begin{equation}
\label{eq:total_loss}
\begin{aligned}
\mathcal L(\theta, \phi) = \mathcal L_{\mathrm{VFM}}(\theta, \phi) + \alpha \mathcal L_{\mathrm{S}}(\theta, \phi)
\end{aligned}
\end{equation}
where the hyperparameter $\alpha$ controls the relative strength of the two components of the loss. See the hyperparameter $\alpha, \beta$ analysis in the Experiment section.
Here, the straightness objective is incorporated into the VFM objective to learn straight generation trajectories, and the latent code $z$ captures the ``generation overview".

In the inference stage, a single latent code $z$ sampled from the prior $p(z)$ is used throughout the generation process from time $t=0$ to $t=1$, akin to Eq.~\eqref{eq:ODE_simulation}:
\begin{equation}
\label{eq:ODE_simulation_VFM}
X_{t_{i+1}} = X_{t_i} + (t_{i+1} - t_i) \, v_\theta(X_{t_i}, t_i, z)
\end{equation}
Since straight trajectories are learned, fewer simulation steps are required compared with Flow Matching in Eq.~\eqref{eq:ODE_simulation}.

\section{Experiment}
\label{sec:experiment}

\subsection{Synthetic Data}

\begin{figure*}[htbp]
  \centering
  \setlength{\tabcolsep}{0pt}
  \begin{tabular}{ccc}
    \subcaptionbox{Flow Matching\strut}[0.3\linewidth]{%
      \includegraphics[width=\linewidth]
      {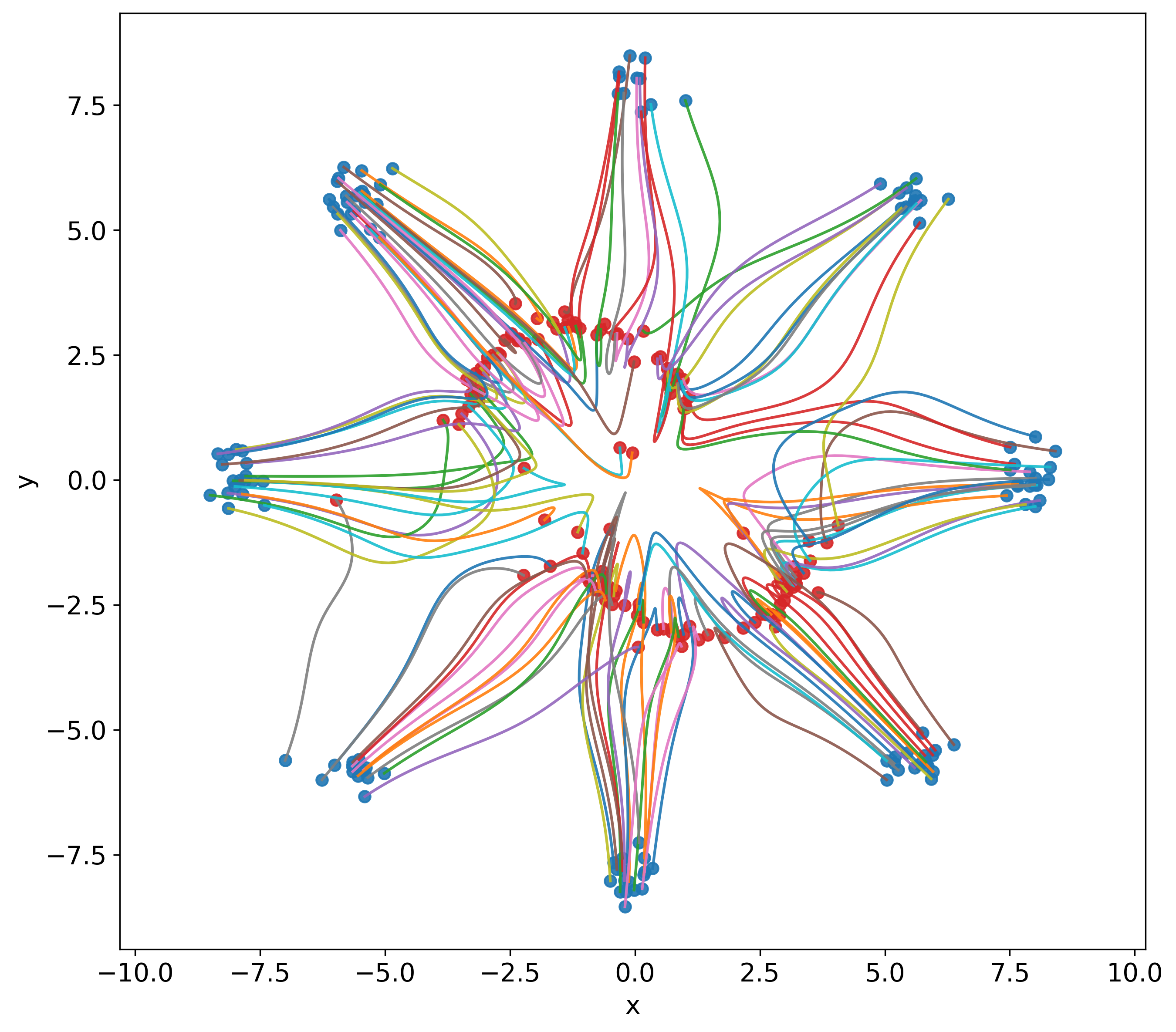}} &
    \subcaptionbox{Rectified Flow\strut}[0.3\linewidth]{%
      \includegraphics[width=\linewidth]{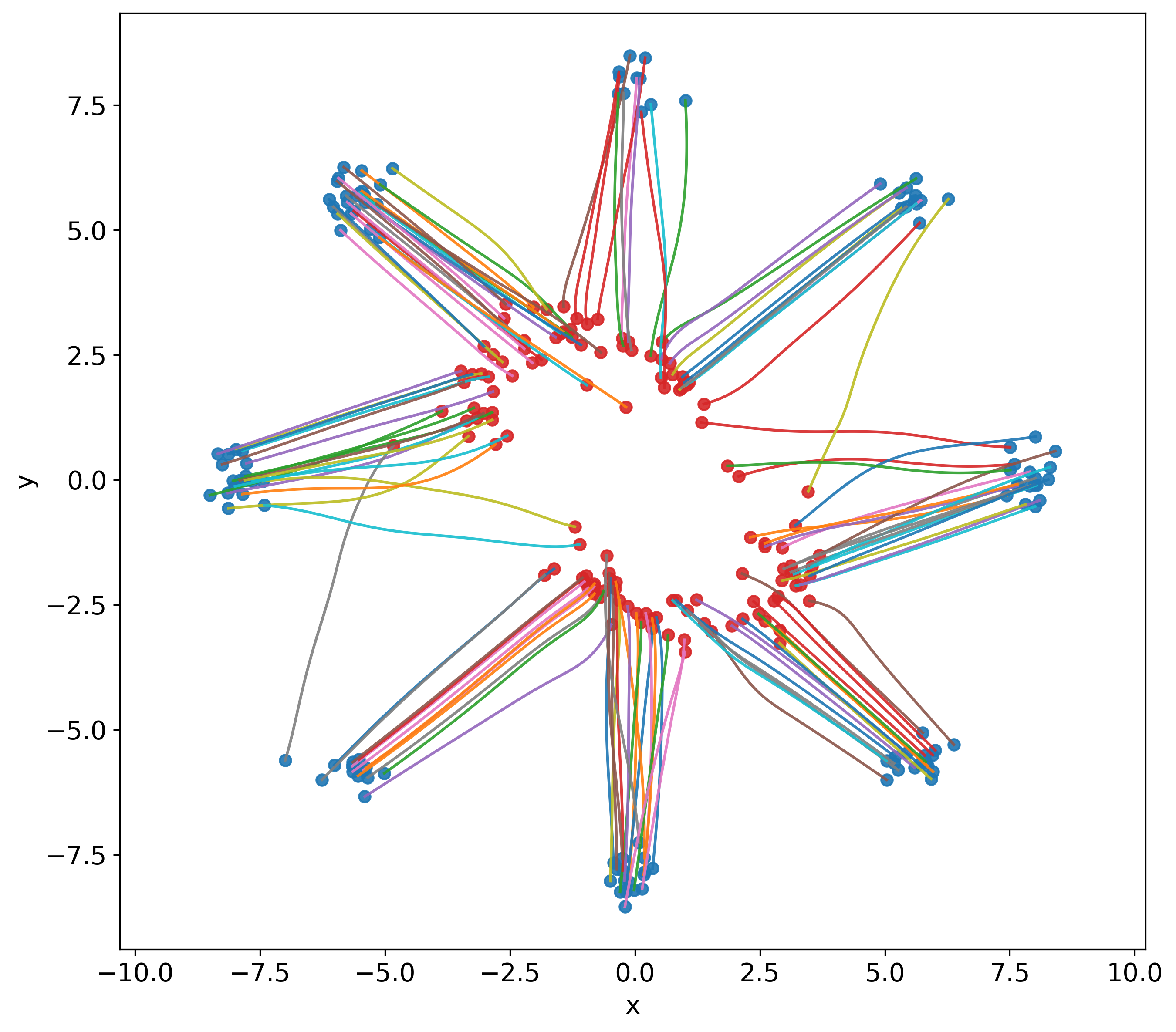}} &
    \subcaptionbox{VFM\strut}[0.3\linewidth]{%
      \includegraphics[width=\linewidth]{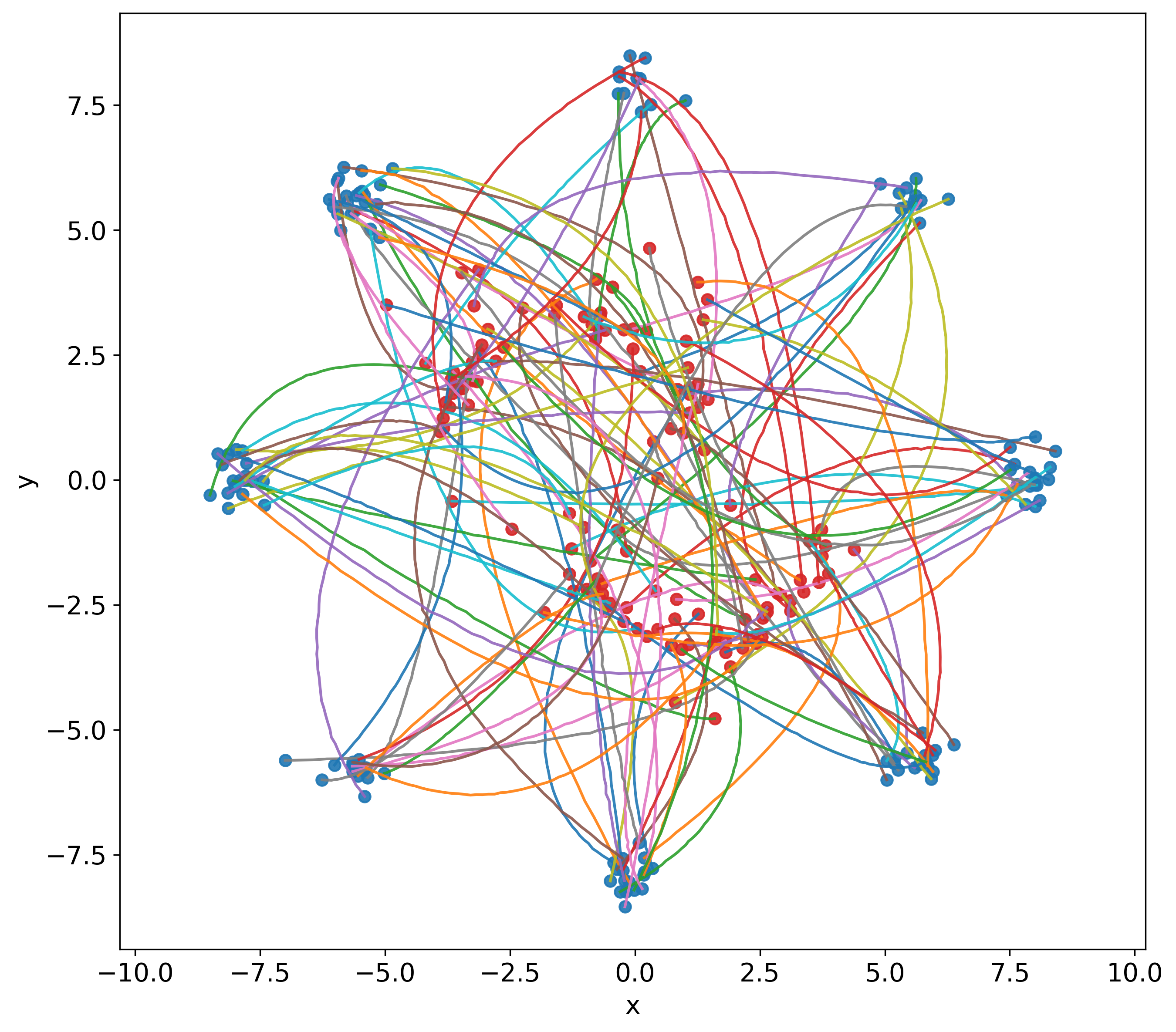}} \\[0pt]
    \subcaptionbox{Consistency Model\strut}[0.3\linewidth]{%
      \includegraphics[width=\linewidth]{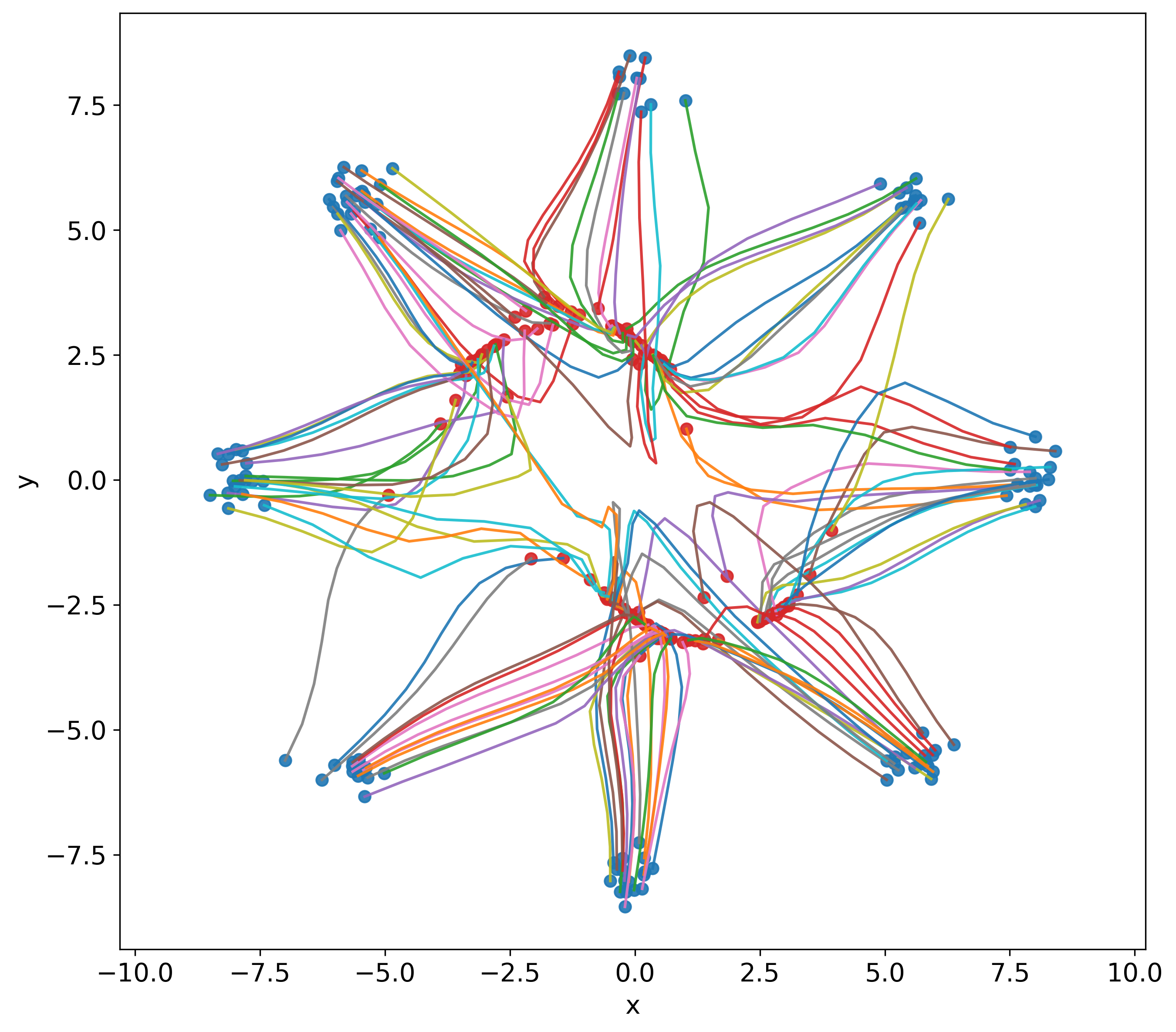}} &
    \subcaptionbox{MeanFlow\strut}[0.3\linewidth]{%
      \includegraphics[width=\linewidth]{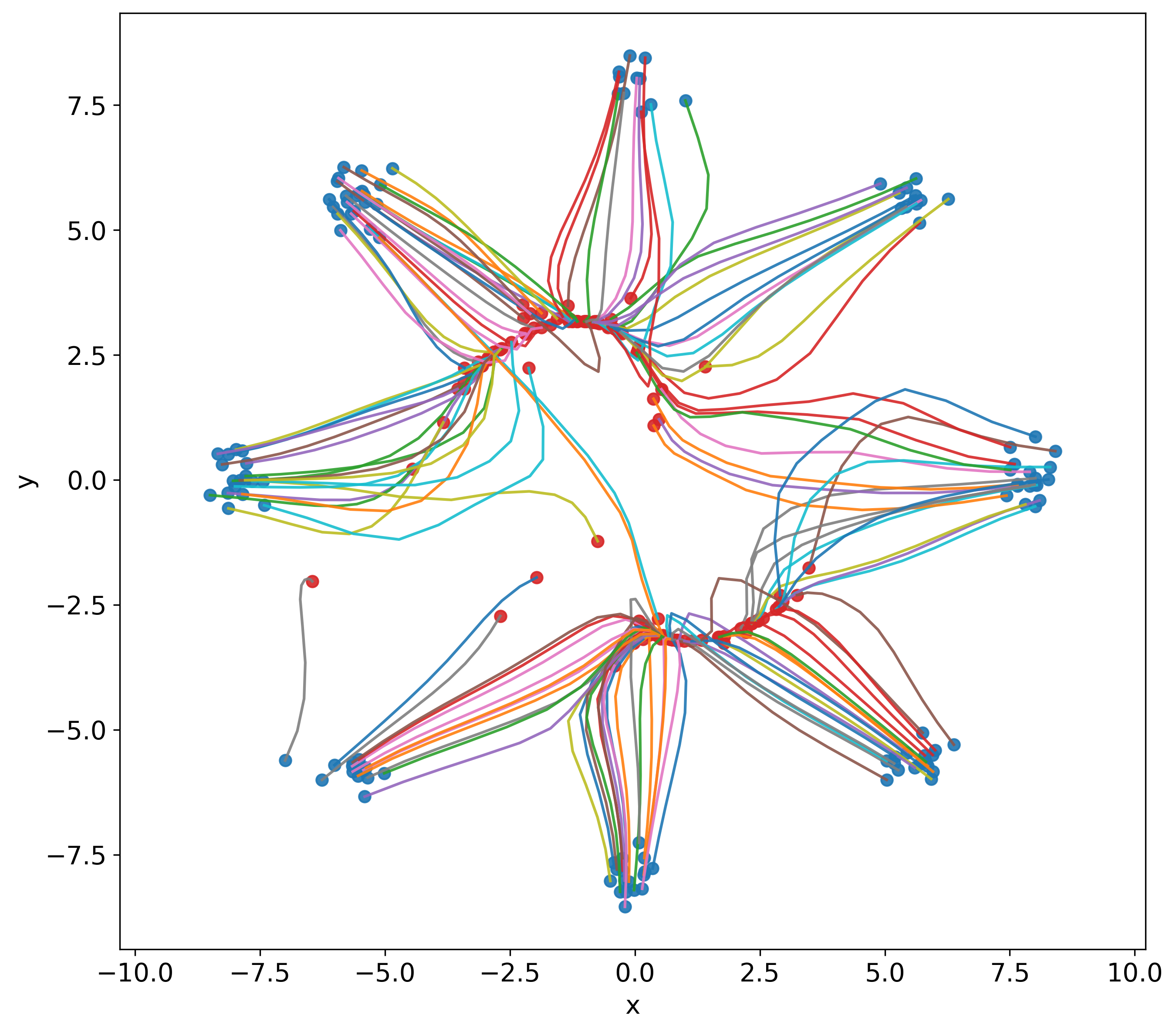}} &
    \subcaptionbox{S-VFM (Ours)\strut}[0.3\linewidth]{%
      \includegraphics[width=\linewidth]{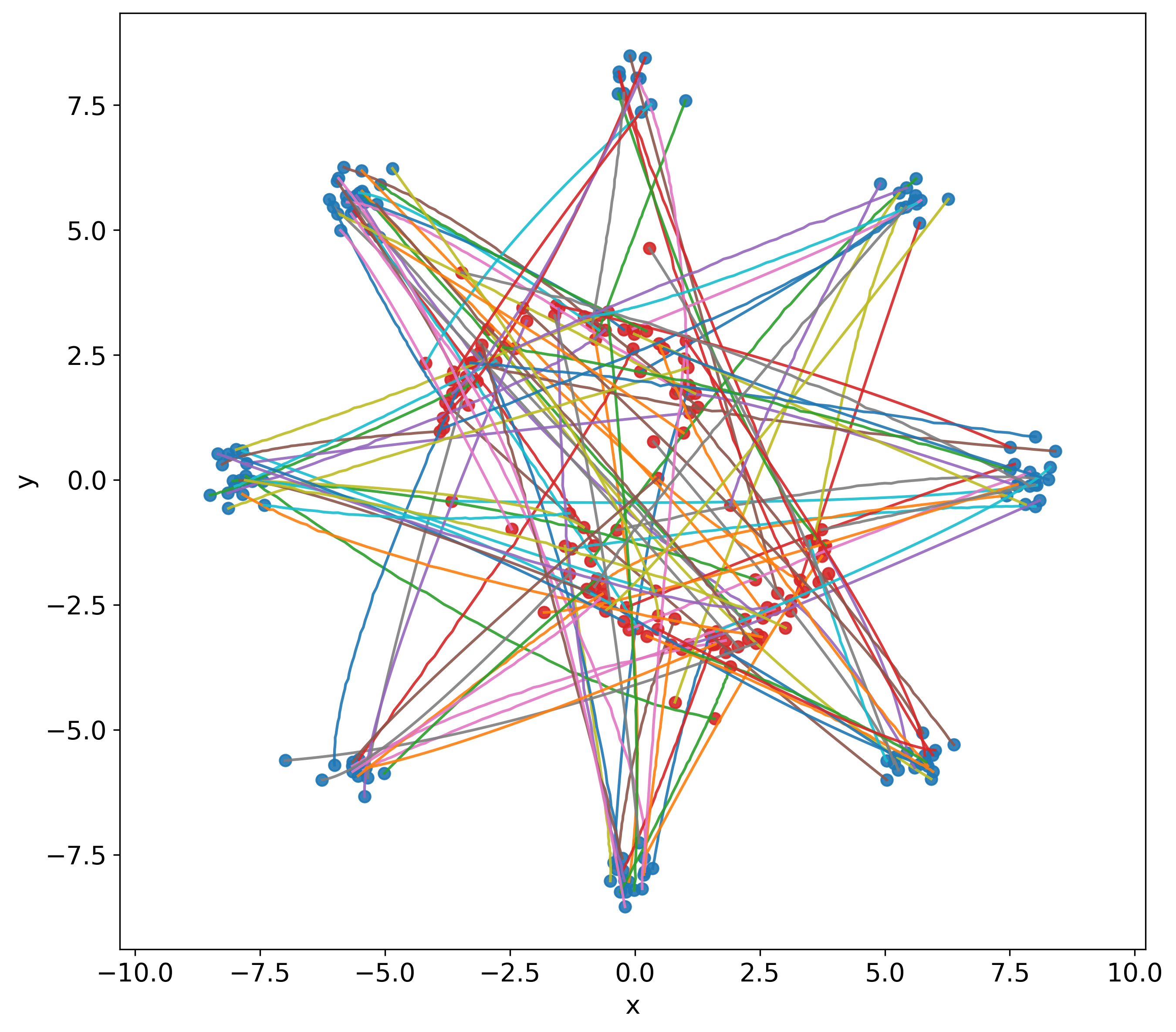}} 
  \end{tabular}
  \vspace{-1em}
  \caption{\textbf{Generation Trajectory Visualization in the 2D Synthesized Eight-Gaussians-to-Moon Dataset.}}
  \label{fig:viz_2D_moon}
\end{figure*}

Synthetic data experiments are conducted to visualize the results and intuitively demonstrate the effectiveness of the proposed method.
First, we simulate and visualize the generation trajectories of different methods on a 2D hexagonal dataset, as shown in Figure~\ref {fig:viz_2D_hex}. In this dataset, the source distribution (blue points) is a simple Gaussian, while the target distribution (red points) consists of six Gaussian clusters arranged around the source. The visualization results are consistent with the discussion in the previous section: our goal is to learn a straight flow even when the trajectories intersect, which is achieved thanks to the introduction of the VFM framework and straightness objective.
Furthermore, we conduct experiments on a more complex synthetic dataset, as illustrated in Figure~\ref {fig:viz_2D_moon}. In this case, the source distribution (blue points) consists of eight Gaussian clusters surrounding the target distribution (red points), which forms an upper–lower moon shape. These results further demonstrate the robustness and generalization capability of our method in modeling complex nonlinear transformations.

\subsection{CIFAR-10}

\begin{figure*}[htbp]
  \centering
  \setlength{\tabcolsep}{0pt} 
  \begin{tabular}{cccc}
    \subcaptionbox{$X^1_0$, $z^1$, NFE = 1\strut}[0.248\linewidth]{%
      \includegraphics[width=\linewidth]{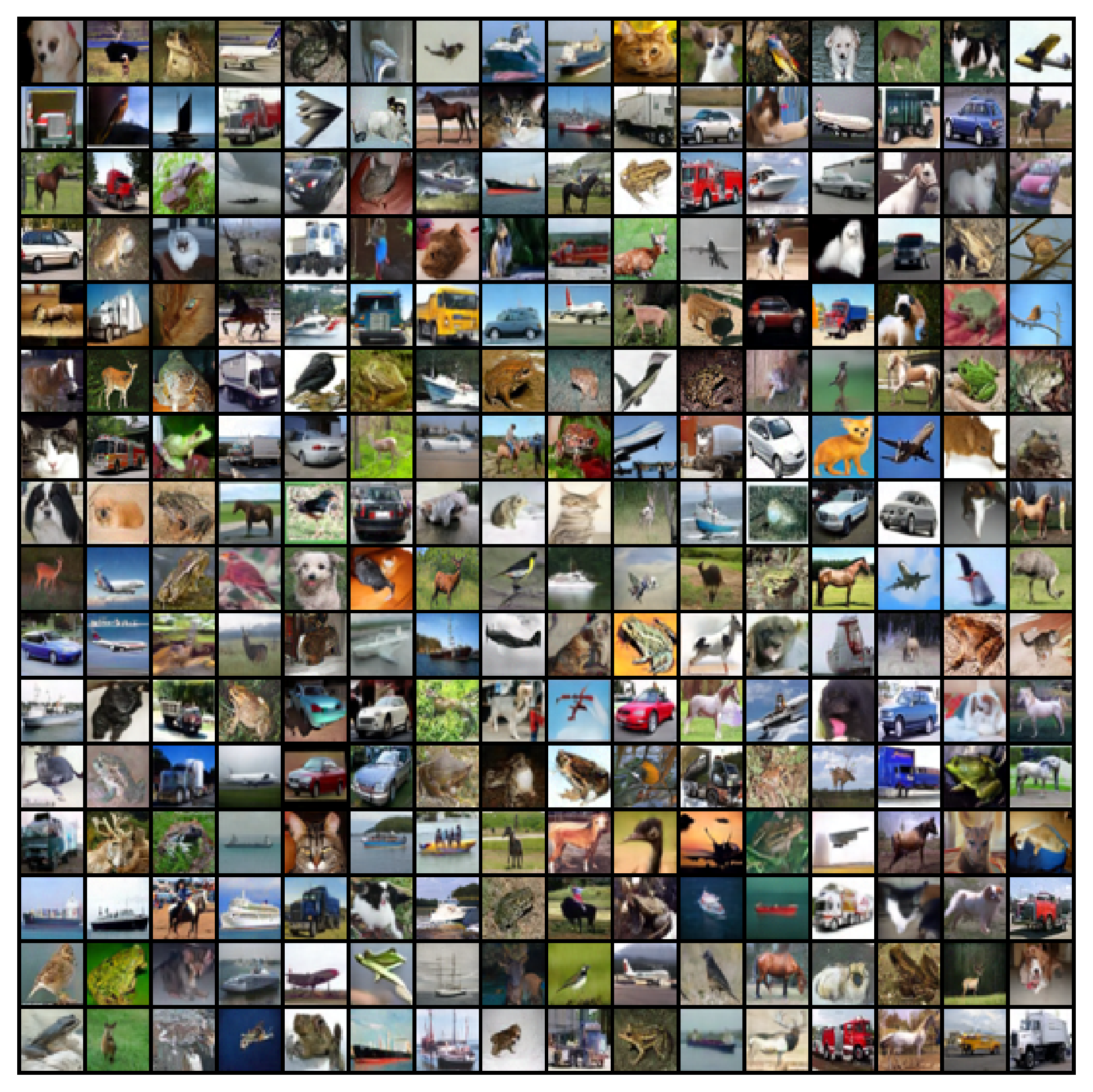}} &
    \subcaptionbox{$X^1_0$, $z^1$, NFE = 2\strut}[0.248\linewidth]{%
      \includegraphics[width=\linewidth]{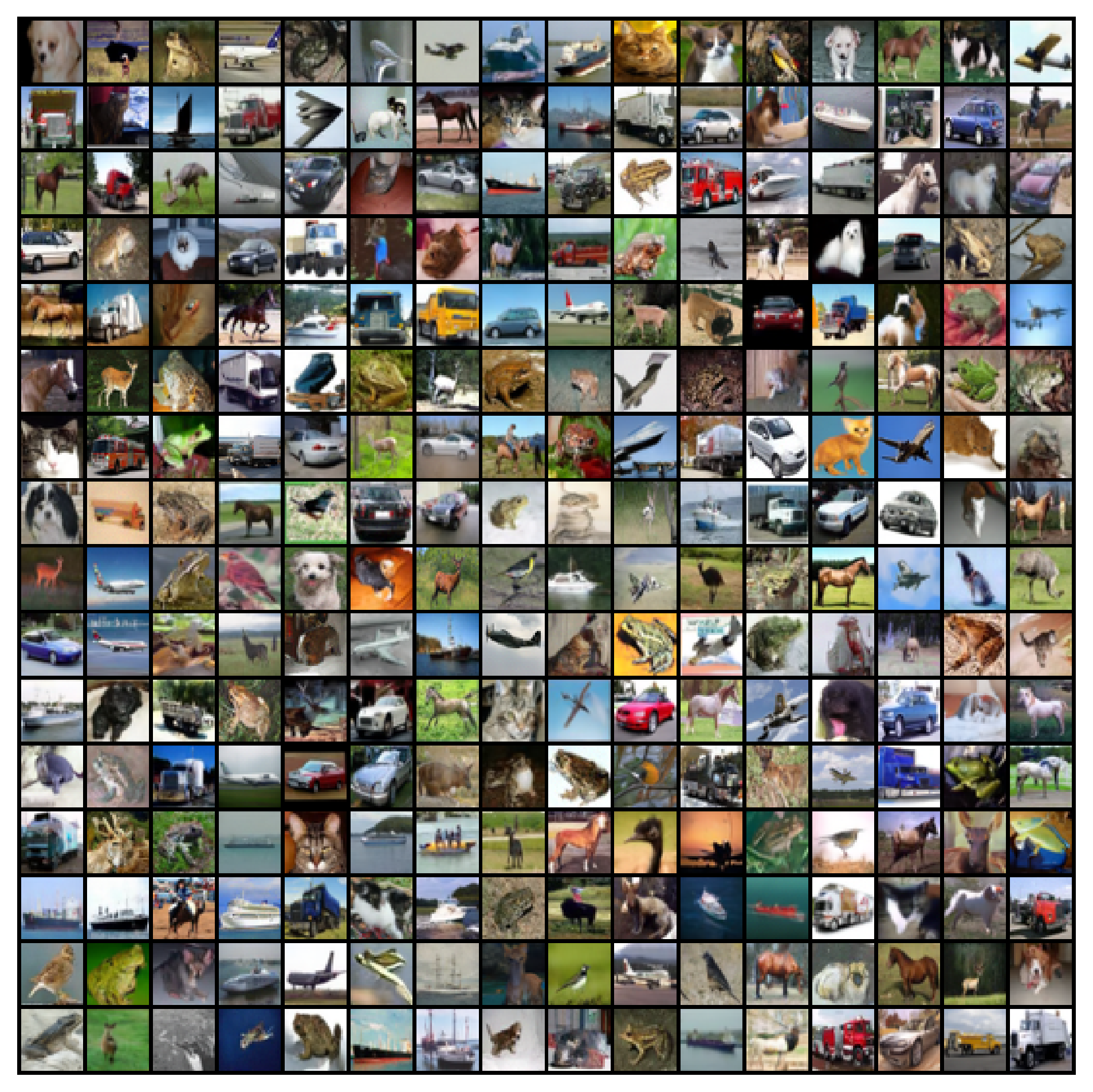}} &
    \subcaptionbox{$X^1_0$, $z^1$, NFE = 5\strut}[0.248\linewidth]{%
      \includegraphics[width=\linewidth]{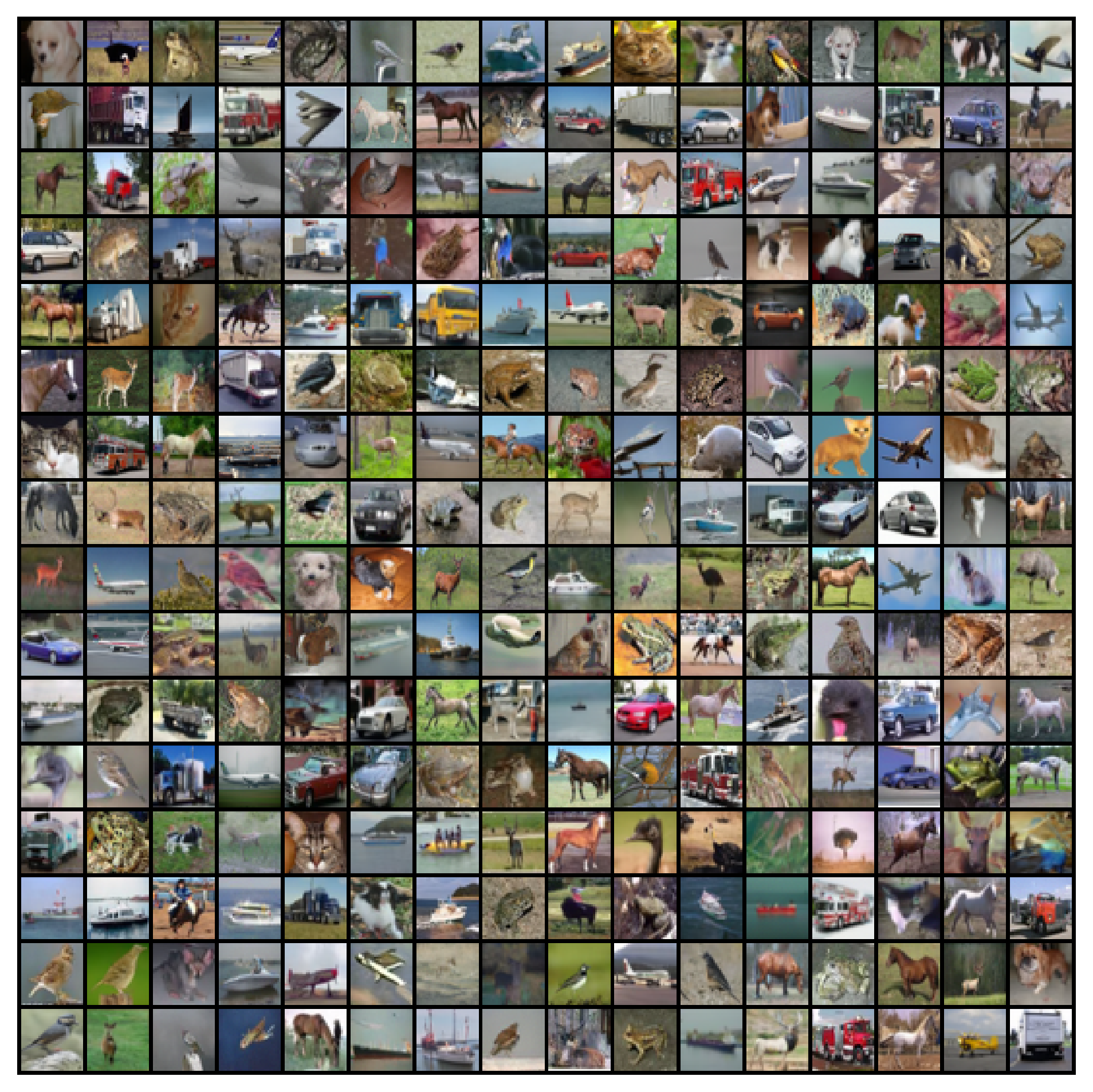}} &
    \subcaptionbox{$X^1_0$, $z^1$, NFE = 10\strut}[0.248\linewidth]{%
      \includegraphics[width=\linewidth]{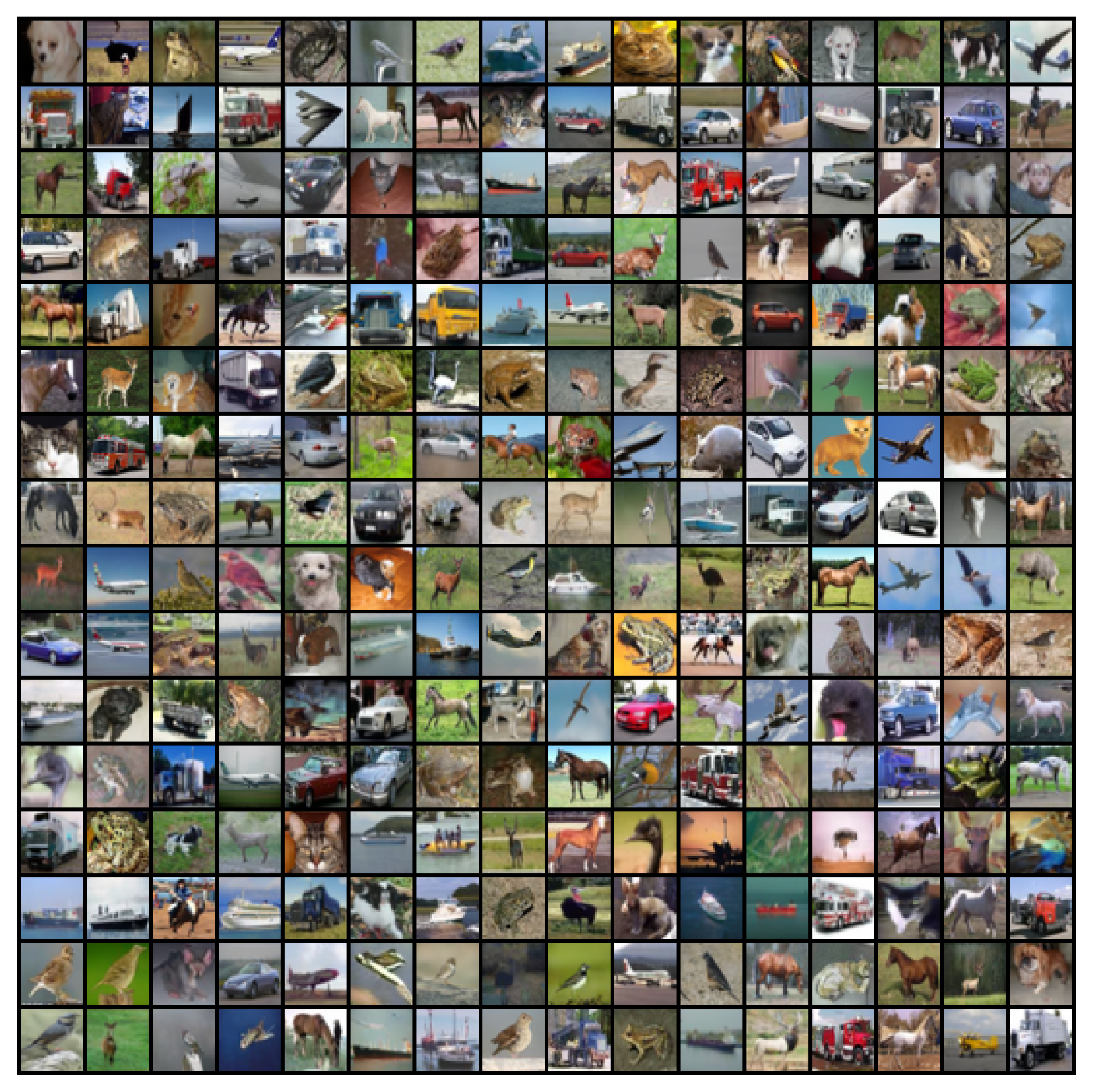}}
  \end{tabular}

  \vspace{2pt} 

  \begin{tabular}{cccc}
    \subcaptionbox{$X^2_0$, $z^2$, NFE = 1\strut}[0.248\linewidth]{%
      \includegraphics[width=\linewidth]{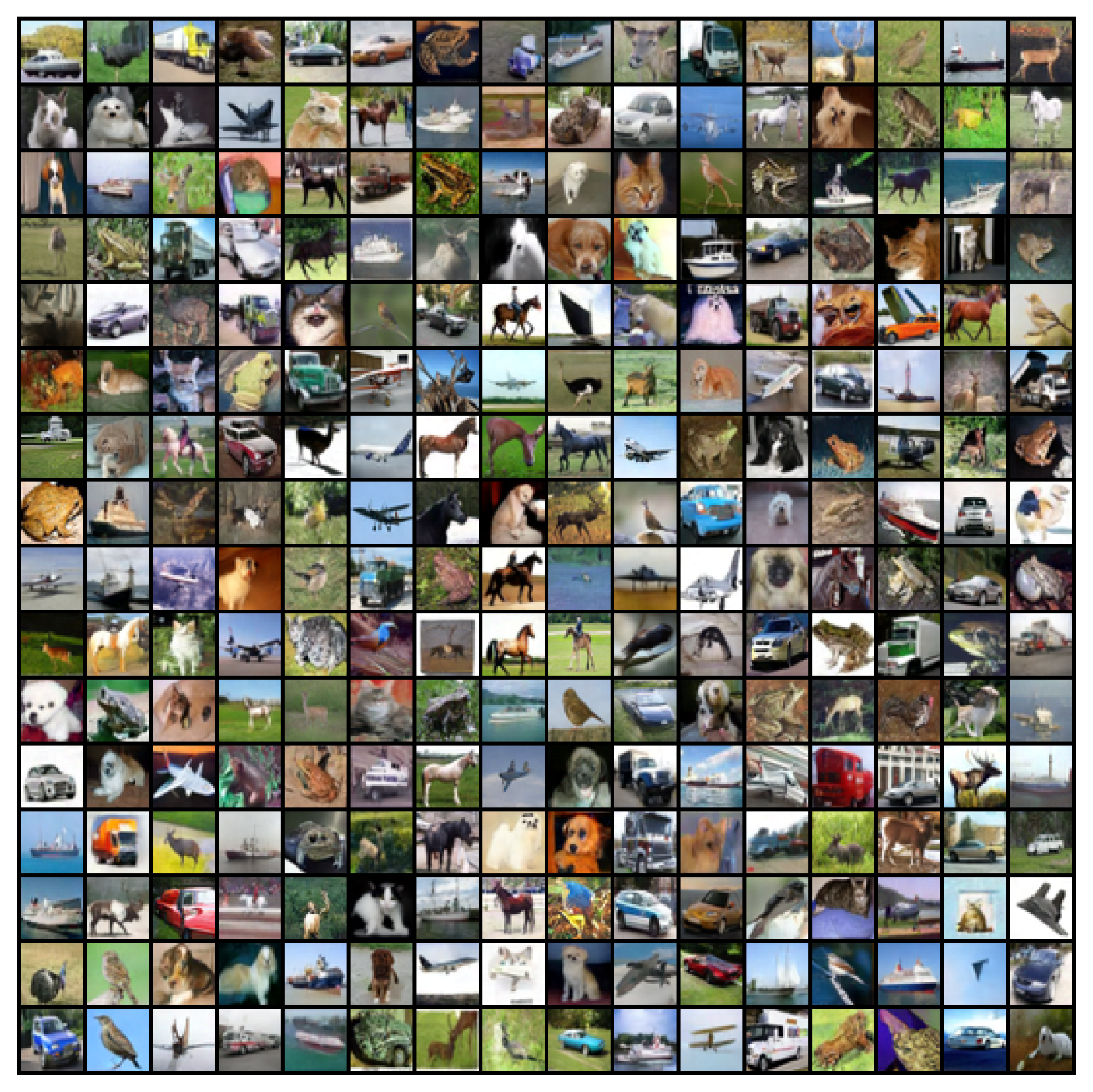}} &
    \subcaptionbox{$X^2_0$, $z^2$, NFE = 2\strut}[0.248\linewidth]{%
      \includegraphics[width=\linewidth]{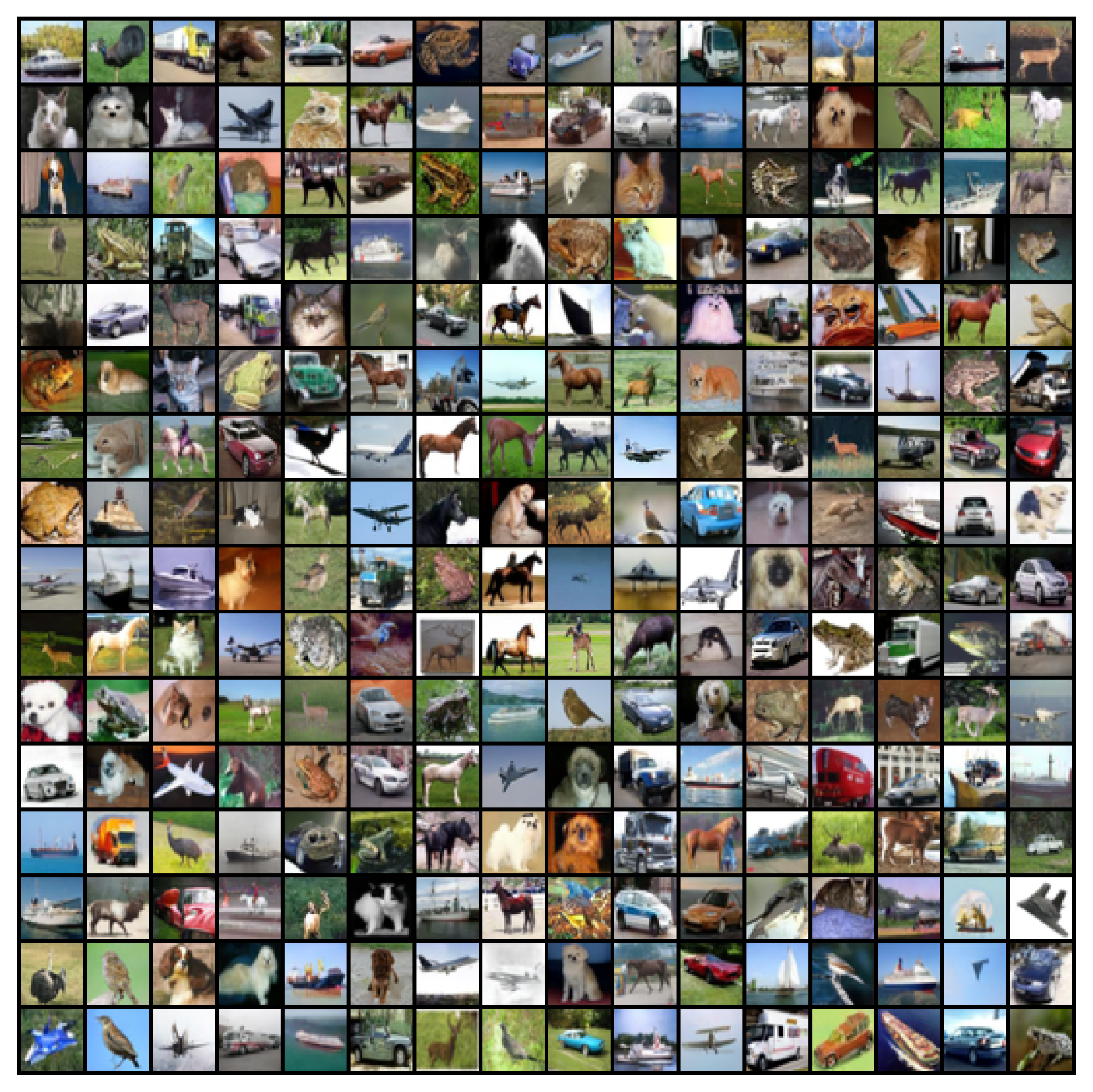}} &
    \subcaptionbox{$X^2_0$, $z^2$, NFE = 5\strut}[0.248\linewidth]{%
      \includegraphics[width=\linewidth]{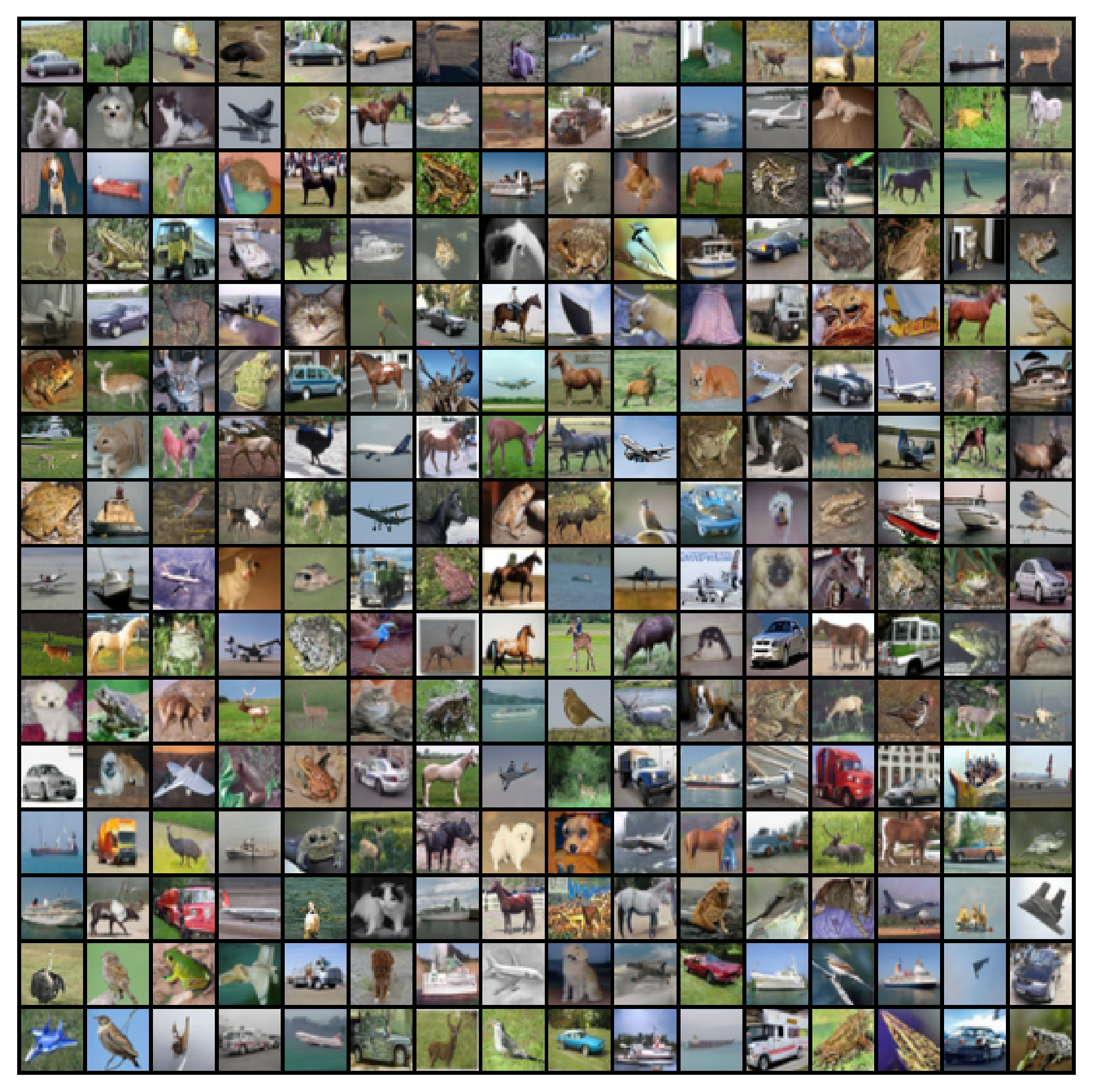}} &
    \subcaptionbox{$X^2_0$, $z^2$, NFE = 10\strut}[0.248\linewidth]{%
      \includegraphics[width=\linewidth]{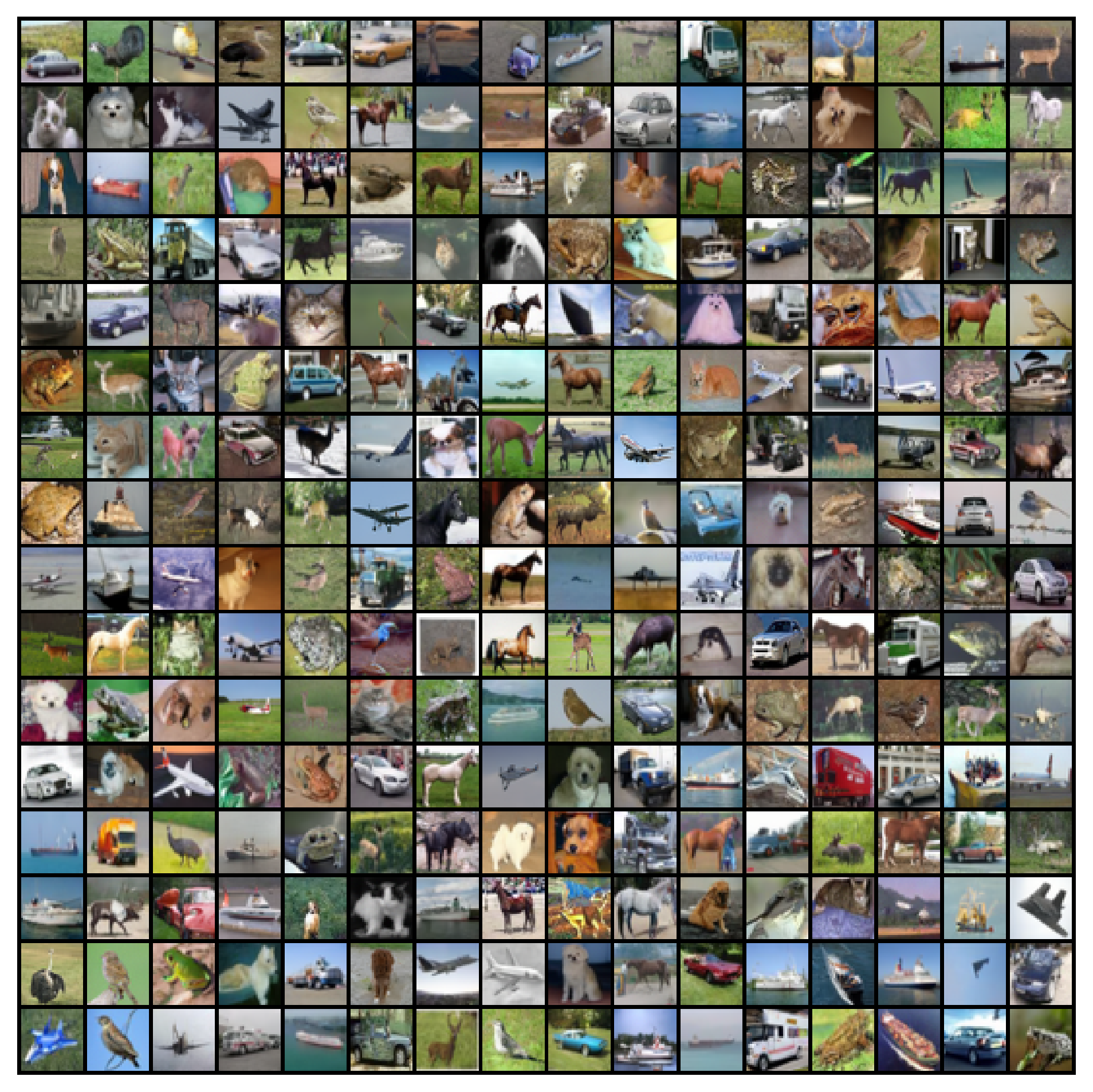}}
  \end{tabular}
\vspace{-1em}
\caption{\textbf{Randomly Selected Generation Results under Different NFE.} Each row corresponds to a distinct initial noise set ($X^1_0$ or $X^2_0$) and its associated latent code set ($z^1$ or $z^2$). 
Within each row, images at the same grid position across panels are generated from the same initial noise and latent code, while panels from left to right correspond to increasing NFE values of $[1, 2, 5, 10]$. Both the noise samples and latent codes are independently drawn from their prior distributions. 
}
  \label{fig:cifar10_NFE}
\end{figure*}

CIFAR-10 is a $32 \times 32$ resolution image dataset containing multiple classes, and is a widely used benchmark in prior work~\cite{krizhevsky2009learning}. For a fair evaluation, we use the same architecture and training paradigm of previous studies~\cite{guo2025variational, geng2025consistency}, but train the UNet model with the VFM framework with the straightness objective, as shown in Eq.~\eqref{eq:total_loss}.

The UNet $v_\theta$ consists of downsampling and upsampling residual blocks with skip connections, and a self-attention block added after the residual block at 16×16 resolution and in the middle bottleneck layer. The model takes both $X_t$, $t$ and $z$ as input, with the time embedding $t$ used to regress learnable scale and shift parameters for adaptive group norm layers.

The posterior model $q_\phi$ shares a similar encoder structure as $v_\theta$: image space inputs $[X_0,X_1,X_t]$ are concatenated along the channel dimension, while time $t$ is conditioned using adaptive group normalization. The network predicts $\mu_\phi$ and $\sigma_\phi$ with dimensions 768.
During training, the variational latent z is sampled from the predicted posterior $q_\phi$, and at test time, from a standard Gaussian prior $p(z)$. The latent is processed through two MLP layers and serves as a conditional signal for the velocity network $v_\theta$. We identify two effective approaches as conditioning mechanisms: \textit{adaptive normalization}, where $z$ is added to the time embedding before computing shift and offset parameters, and \textit{bottleneck sum}, which fuses the latent with intermediate activations at the lowest resolution using a weighted sum before upsampling.

\begin{figure*}[htbp]
  \centering
  \setlength{\tabcolsep}{1pt} 
  \begin{tabular}{cccc}
    \subcaptionbox{$X^1_0$, $z^1$, NFE = 1\strut}[0.248\linewidth]{%
      \includegraphics[width=\linewidth]{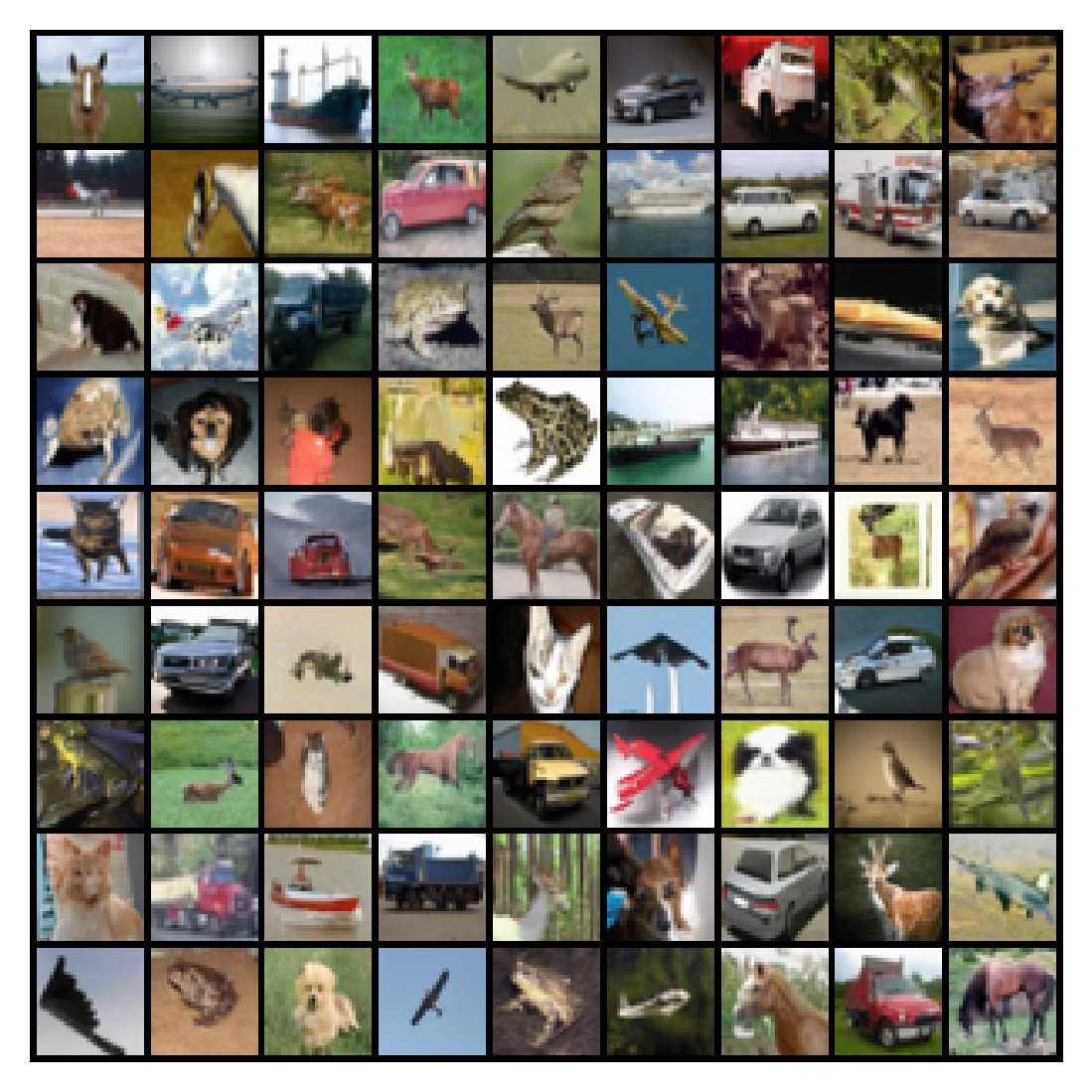}} &
    \subcaptionbox{$X^1_0$, $z^2$, NFE = 1\strut}[0.248\linewidth]{%
      \includegraphics[width=\linewidth]{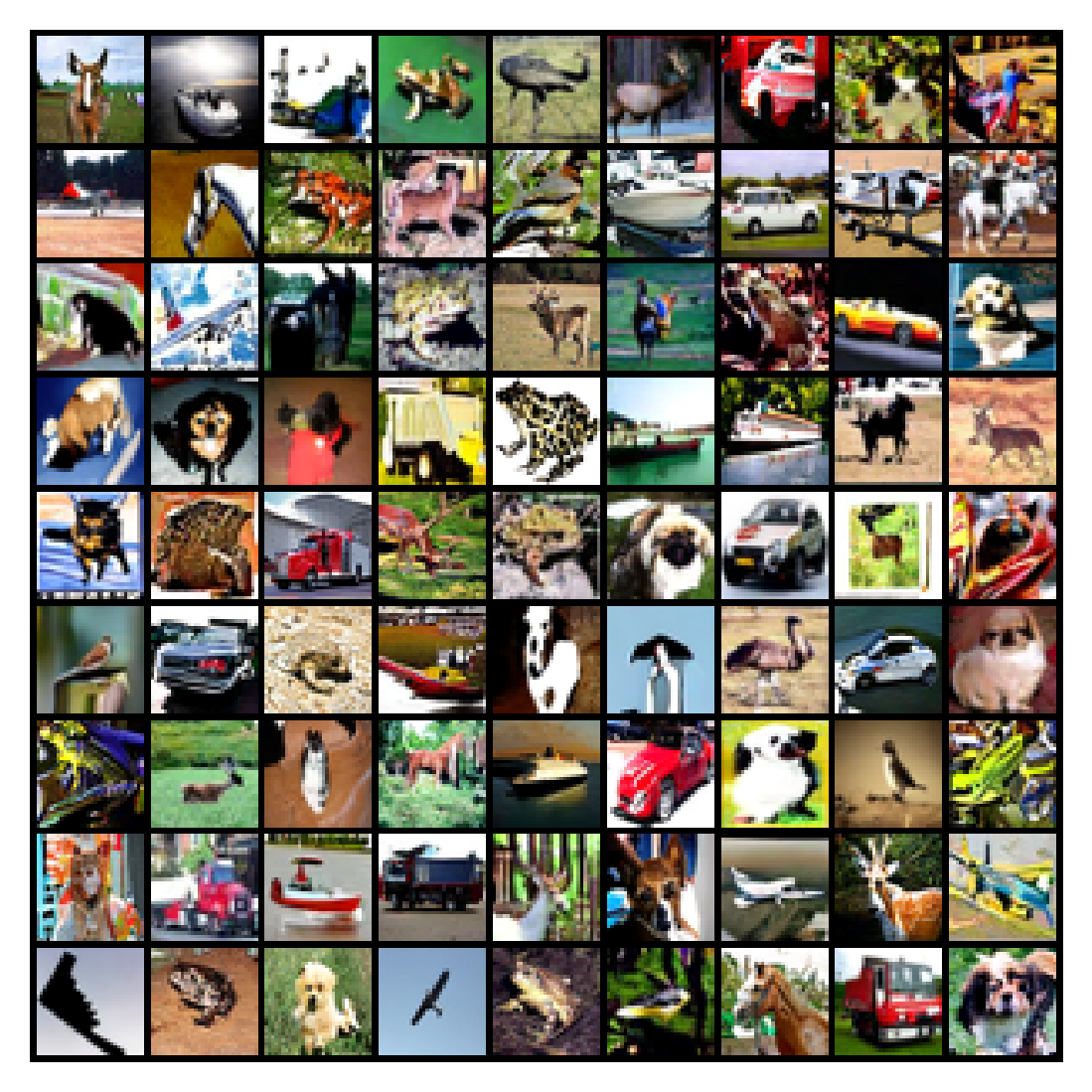}} &
    \subcaptionbox{$X^2_0$, $z^3$, NFE = 1\strut}[0.248\linewidth]{%
      \includegraphics[width=\linewidth]{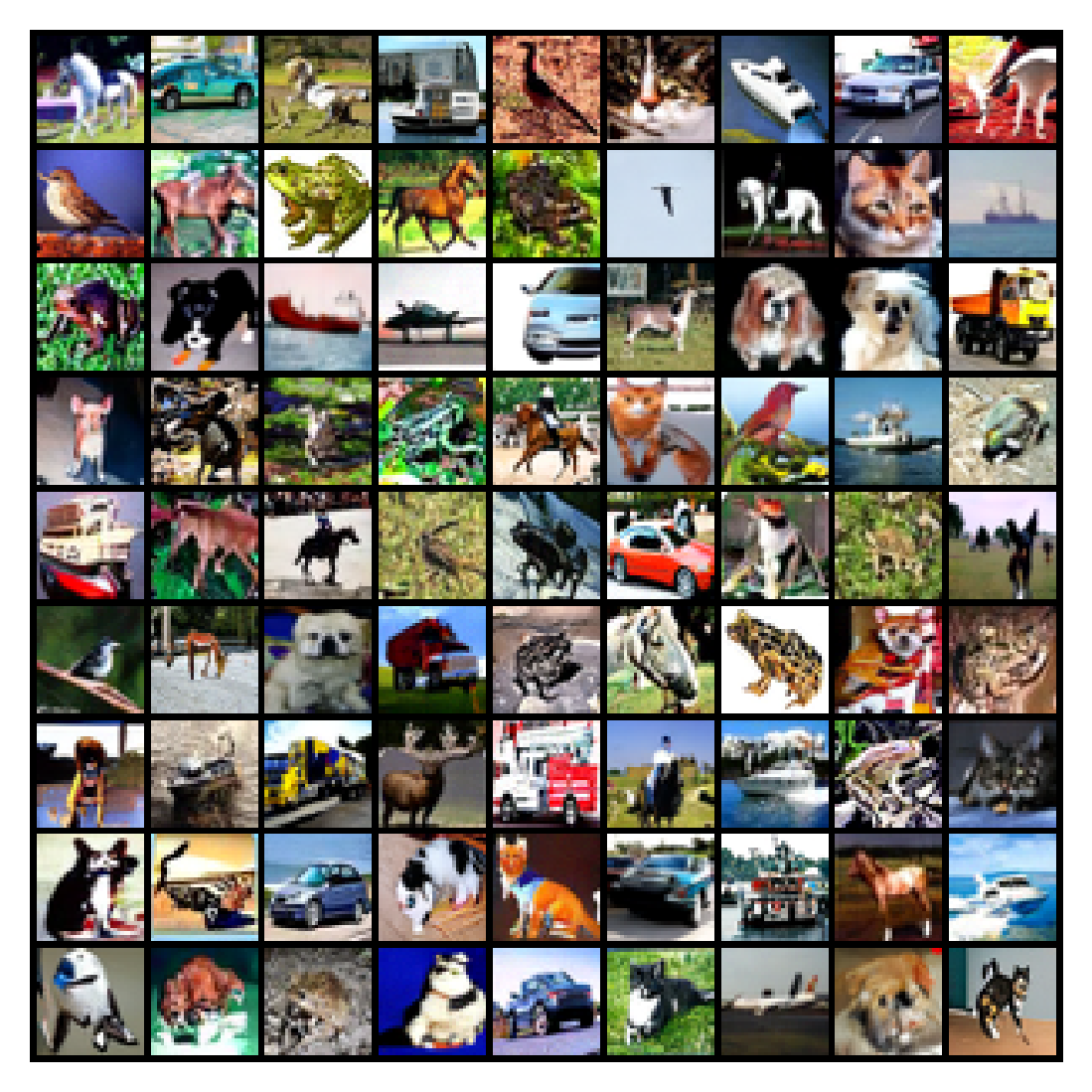}} &
    \subcaptionbox{$X^2_0$, $z^4$, NFE = 1\strut}[0.248\linewidth]{%
      \includegraphics[width=\linewidth]{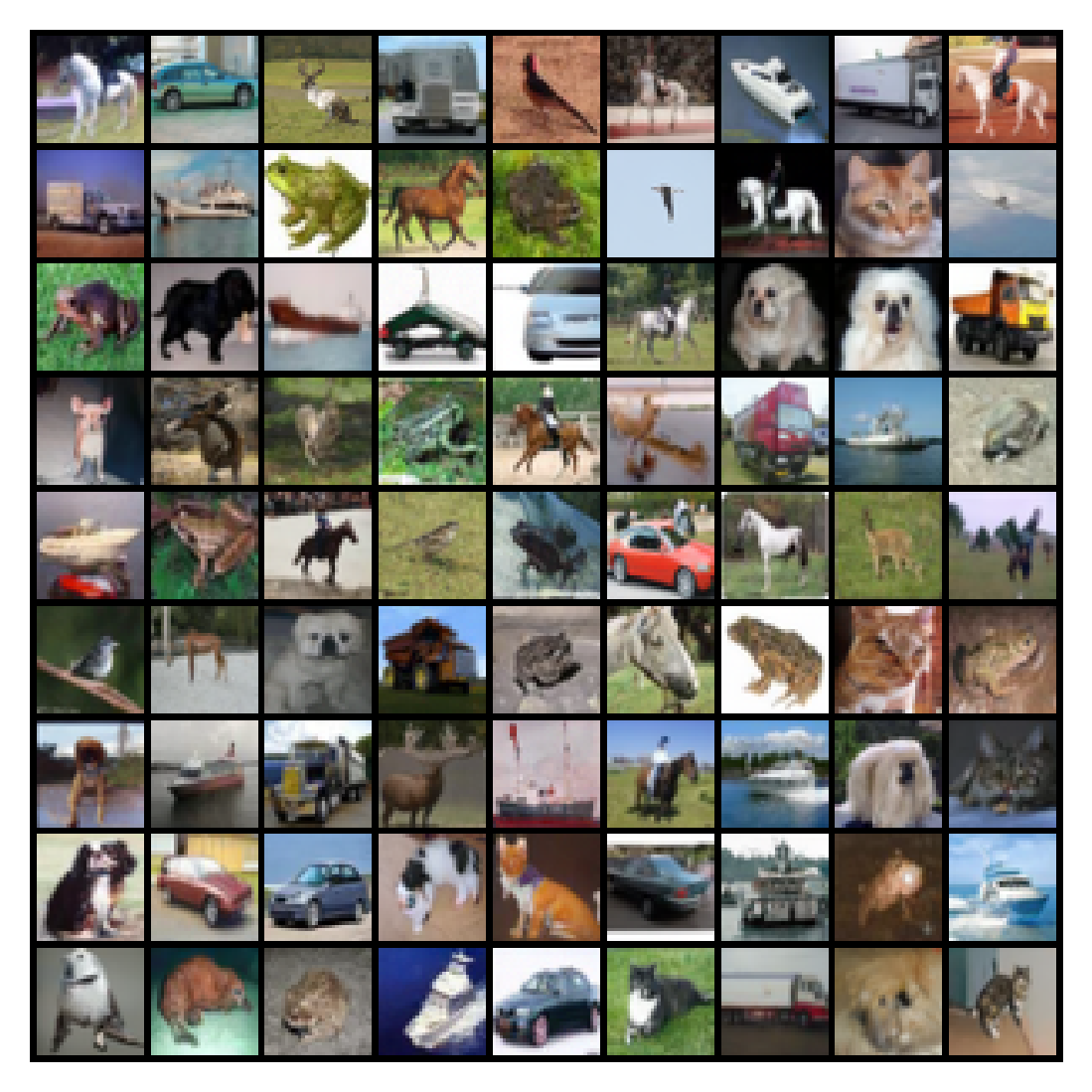}}
  \end{tabular}
  \vspace{-1em}
\caption{
\textbf{Generation results under the same initial noise but different latent codes.} 
Panels (a) and (b) are generated from the same initial noise set $X^1_0$ with different latent code sets $z^1$ and $z^2$, respectively, while panels (c) and (d) use a different initial noise set $X^2_0$ with corresponding latent code sets $z^3$ and $z^4$. 
All noise samples and latent codes are independently drawn from their prior distributions, and the number of function evaluations is fixed at $\text{NFE}=1$. 
}

  \label{fig:cifar10_different_z}
\end{figure*}

Figure~\ref{fig:cifar10_NFE} illustrates the generation results under different numbers of function evaluations (NFE), using two distinct initial noise sets, $X^1_0$ and $X^2_0$, along with their corresponding latent code sets, $z^1$ and $z^2$, respectively. 
Note that the noise samples and latent codes within each set are randomly generated following their prior distributions, as described in the Method section, and are independently sampled. 
Within each row, images at the same grid position across panels are generated from the same initial noise (belonging to $X^1_0$ or $X^2_0$) and the corresponding latent code (from $z^1$ or $z^2$). 
Each image panel in a row varies by the number of function evaluations, with NFE values of $[1, 2, 5, 10]$ from left to right. 
As the NFE increases, both the visual fidelity and structural consistency of the generated images improve. 
Notably, even the one-step generation (NFE $=1$) demonstrates competitive quality compared to multi-step generations, which is a key observation supporting the claim that the proposed method yields \textit{nearly straight generation trajectories} to support few-steps generation.

Figure~\ref{fig:cifar10_NFE} also demonstrates the effect of generation from the same latent code under various initial noises: each image panel contains samples generated from different initial noises while keeping the latent code fixed. 
Furthermore, to investigate the generative behavior under the same initial noise but different latent codes, we visualize two groups of images in Figure~\ref{fig:cifar10_different_z}. 
Images at the same grid position in panels (a) and (b) are generated from the same initial noise belonging to the noise set $X^1_0$, but with different corresponding latent codes from the latent code sets $z^1$ and $z^2$, respectively. 
The same configuration applies to panels (c) and (d), which differ by the initial noise set $X^2_0$ and latent code sets $z^3$ and $z^4$. 
As in the previous experiment, the noise samples and latent codes within each set are randomly drawn from their respective prior distributions, as described in the Method section, and are independently sampled. 
For all panels, the number of function evaluations is fixed at $\text{NFE}=1$. 
The results shown in Figure~\ref{fig:cifar10_different_z} indicate that images generated from the same initial noise but conditioned on different latent codes exhibit consistent color patterns and similar spatial layouts. 
However, variations in the latent codes can lead to changes in object class or the presence of different instances of the same class.

\begin{table*}[t]
\centering
\scriptsize
\setlength{\tabcolsep}{4pt}
\renewcommand{\arraystretch}{1}

\begin{adjustbox}{width=0.7\textwidth,center}
\begin{tabular}{c l c c c c c c}
\Xhline{4\arrayrulewidth}
 & \textbf{NFE / Sampler} & \textbf{\# Params.} & \textbf{1} & \textbf{2} & \textbf{5} & \textbf{10} & \textbf{Adaptive} \\
\hline\hline
\multirow{2}{*}{\rotatebox[origin=c]{90}{\tiny \makecell[c]{Flow}}}
& {Flow Matching~\cite{lipman2023flow} \textcolor{lightgray}{\scriptsize{[ICLR'23]}}} & 36.5M & - & 166.65 & 36.19 & 14.4 & 3.66 \\
& {VFM~\cite{guo2025variational} \textcolor{lightgray}{\scriptsize{[ICML'25]}}} & 60.6M & - & 97.83 & 13.12 & 5.34 & 2.49 \\
\hline
\multirow{3}{*}{\rotatebox{90}{\tiny \makecell[c]{Rectified\\Flow}}} 
& {1-Rectified Flow~\cite{liu2023flow} \textcolor{lightgray}{\scriptsize{[ICLR'23]}}} & 36.5M & 378 & 6.18 & - & - & 2.58 \\
& {2-Rectified Flow~\cite{liu2023flow} \textcolor{lightgray}{\scriptsize{[ICLR'23]}}} & 36.5M & 12.21 & 4.85 & - & - & 3.36 \\
& {3-Rectified Flow~\cite{liu2023flow} \textcolor{lightgray}{\scriptsize{[ICLR'23]}}} & 36.5M & 8.15 & 5.21 & - & - & 3.96 \\
\hline
\multirow{6}{*}{\rotatebox{90}{\tiny \makecell[c]{Mean Velocity/\\Consistency Models}}}
& {CT~\cite{song2023consistency} \textcolor{lightgray}{\scriptsize{[ICML'23]}}} & 61.8M & 8.71 & 5.32 & 11.412 & 23.948 & 40.677 \\
& {iCT~\cite{song2024improved} \textcolor{lightgray}{\scriptsize{[ICLR'24]}}} & ~55M & 2.83 & 2.46 & - & - & - \\
& {ECT~\cite{geng2025consistency} \textcolor{lightgray}{\scriptsize{[ICLR'25]}}} & ~55M & 3.60 & 2.11 & - & - & - \\
& {sCT~\cite{lu2025simplifying} \textcolor{lightgray}{\scriptsize{[ICLR'25]}}} & ~55M & 2.85 & 2.06 & - & - & - \\
& {IMM~\cite{zhou2025inductive} \textcolor{lightgray}{\scriptsize{[ICML'25]}}} & ~55M & 3.20 & 1.98 & - & - & - \\
& {MeanFlow~\cite{geng2025mean} \textcolor{lightgray}{\scriptsize{[NIPS'25]}}} & ~55M & 2.92 & 2.23 & 2.84 & 2.27 & - \\
\hline
\rowcolor{cvprblue!15} 
& {S-VFM (bottleneck sum)\textcolor{lightgray}{\scriptsize{[Ours]}}} & 60.6M & 2.94 & 2.28 & 2.09 & 2.06 & 2.01 \\
\rowcolor{cvprblue!15} 
\multirow{-2}{*}{\rotatebox{90}{\tiny Ours}}
& {S-VFM (adaptive norm)\textcolor{lightgray}{\scriptsize{[Ours]}}} & 60.6M & 2.81 & 2.16 & 2.02 & 1.97 & 1.95 \\
\Xhline{4\arrayrulewidth}
\end{tabular}
\end{adjustbox} 

\vspace{-1em}
\caption{
\textbf{Quantitative Comparison with Different Generation Methods on CIFAR-10 Dataset.} 
Our method (\textit{adaptive normalization} variant) achieves the best performance in one-step generation (NFE $=1$). 
Moreover, the FID score consistently decreases as NFE increases.
}
\label{tab:cifar10_fid}
\end{table*}

\vspace{-1mm}
To quantitatively evaluate the generation performance, we compare our method with several state-of-the-art approaches by measuring the generation quality using the Fréchet Inception Distance (FID)~\cite{heusel2017gans}, computed under varying NFE: $[1, 2, 5, 10]$ and adaptive-step Dopri5 ODE solver~\cite{dormand1980family}, as presented in Table~\ref{tab:cifar10_fid}. 
Two model variants employing different latent code conditioning mechanisms---\textit{adaptive normalization} and \textit{bottleneck sum}---were tested, with the straightness loss weight fixed at $\alpha = 10$ and the KL loss weight at $\beta = 1\text{e}{-2}$. 
See the hyperparameter analysis in the Ablation Study section for further details.
The results in Table~\ref{tab:cifar10_fid} show that our method (\textit{adaptive normalization} variant) achieves the best performance in one-step generation (NFE $=1$). 
Moreover, the FID score consistently decreases as NFE increases, while both Consistency Models and Mean Velocity Models tend to exhibit degraded performance with higher NFE values.

\subsection{ImageNet}
To evaluate robustness and efficiency on large-scale data, we evaluate methods on the ImageNet 256 × 256 dataset~\cite{krizhevsky2012imagenet}. For a fair comparison, we use the SiT-XL~\cite{ma2024sit} architecture, a transformer-based model that has shown strong results in image generation, and has been selected as backbone for many previous studies~\cite{zhang2025creatilayout, jin2025pyramidal, ma2025beyond}.
For a fair comparison, we strictly follow the original training recipe in the open-source SiT repository and replicate the training process from the SiT paper, while introducing our model, S-VFM-XL, by substituting the classic flow matching loss with the straight VFM loss in Equation~\eqref{eq:total_loss}. 
The posterior model $q_\phi$ also utilizes an SiT transformer architecture but with half the number of blocks. In the final layer, the features are average-pooled and passed through an MLP layer to predict $\mu_\phi$ and $\sigma_\phi$. 
We sample the latent code $z$ from the posterior $q_\phi$ during training and from the prior distribution $p(z)$ during inference. This latent code $z$ is then processed by three MLP layers and fused with the velocity network $v_\theta$ via adaptive normalization. 
By default, we use the Euler-Maruyama sampler with the SDE solver in the inference stage, as described by SiT~\cite{ma2024sit}.
\begin{table}[t]
\centering
\scriptsize
\setlength{\tabcolsep}{5pt}
\renewcommand{\arraystretch}{1}

\begin{adjustbox}{width=0.9\linewidth,center}
\begin{tabular}{lcccc}
\Xhline{4\arrayrulewidth} 
\textbf{Method} & \textbf{\# Params.} & \textbf{NFE} & \textbf{FID} \\
\hline\hline 
iCT-XL/2~\cite{song2024improved} \textcolor{lightgray}{\scriptsize{[ICLR'24]}} & 675M & 1 & 34.24 \\
Shortcut-XL/2~\cite{frans2025one} \textcolor{lightgray}{\scriptsize{[ICLR'25]}} & 675M & 1 & 10.60 \\
MeanFlow-XL/2~\cite{geng2025mean} \textcolor{lightgray}{\scriptsize{[NIPS'25]}} & 676M & 1 & 3.43 \\
\rowcolor{cvprblue!15}
S-VFM-XL/2 \textcolor{lightgray}{\scriptsize{[Ours]}} & 677M & 1 & 3.31 \\
\hline 
iCT-XL/2~\cite{song2024improved} \textcolor{lightgray}{\scriptsize{[ICLR'24]}} & 675M & 2 & 20.30 \\
iMM-XL/2~\cite{zhou2025inductive} \textcolor{lightgray}{\scriptsize{[ICML'25]}} & 675M & $1\times2$ & 7.77 \\
MeanFlow-XL/2~\cite{geng2025mean} \textcolor{lightgray}{\scriptsize{[NIPS'25]}} & 676M & 2 & 2.93 \\
\rowcolor{cvprblue!15}
S-VFM-XL/2 \textcolor{lightgray}{\scriptsize{[Ours]}} & 677M & 2 & 2.86 \\
\Xhline{4\arrayrulewidth} 
\end{tabular}
\end{adjustbox} 

\vspace{-1em}
\caption{\textbf{Quantitative Comparison with Different Generation Methods on ImageNet $256 \times 256$ Dataset.} 
Our method achieves the best performance in few-step generation.}
\label{tab:image256}
\end{table}

Following the evaluation protocol of previous studies~\cite{geng2025mean, song2024improved}, we randomly generate 50K images from each model and report the corresponding FID scores in Table~\ref{tab:image256}. 
S-VFM-XL consistently outperforms both Consistency Models and Mean Velocity Models, achieving notable gains under identical training conditions. 
These results highlight the importance of modeling a global \textit{generation overview} within the velocity vector field, which helps resolve ambiguities caused by intersecting interpolants and complements the straightness objective, resulting in straighter generation trajectories that enable few-step generation. 
\begin{figure}[htbp]
  \centering
  \includegraphics[width=0.8\linewidth]{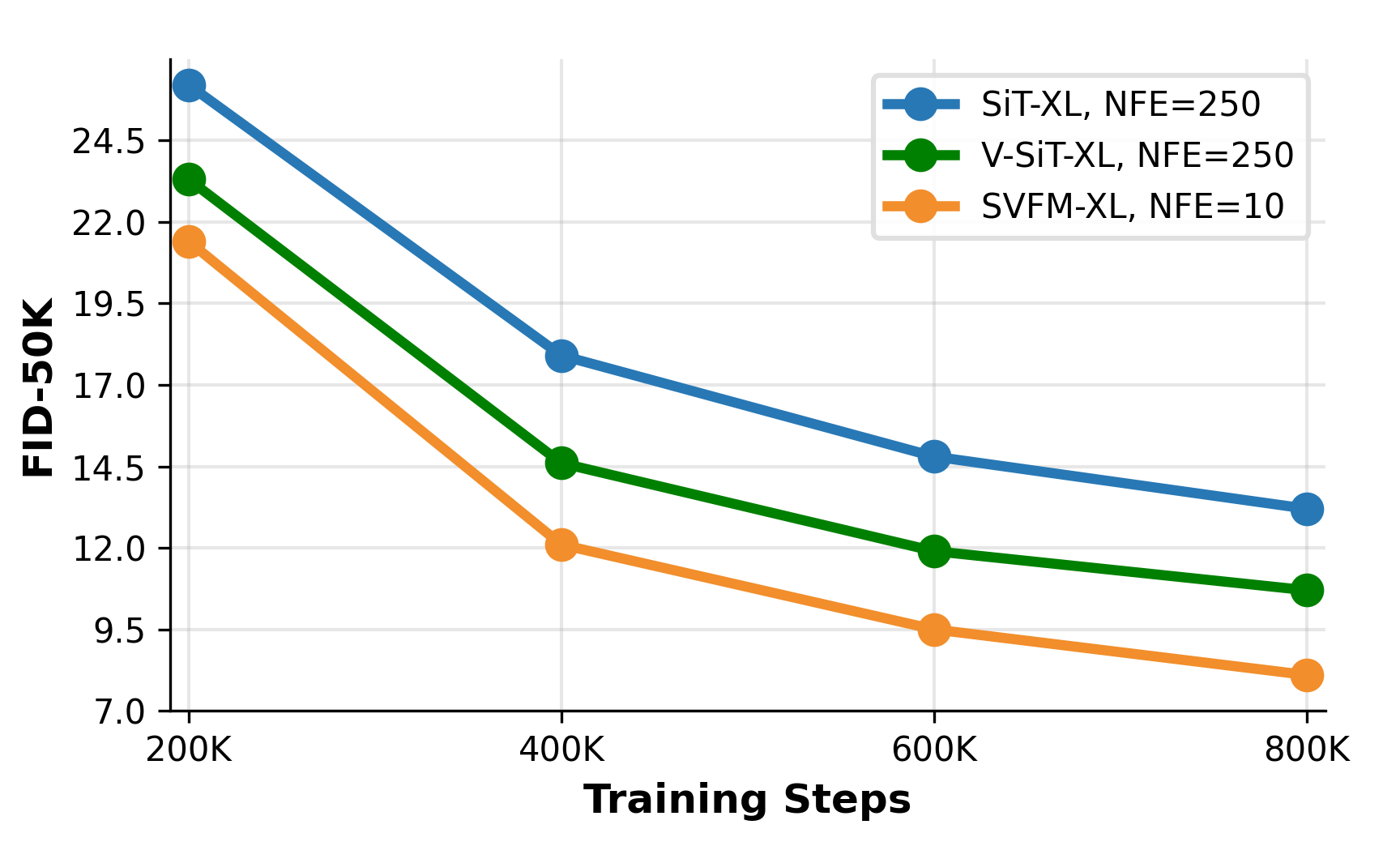}
    
  \vspace{-1em}
\caption{
\textbf{Comparison of FID-50K Score over Training Iterations on ImageNet $256 \times 256$ Dataset.} 
}

  \label{fig:imagenet_train}
\end{figure}

In addition, we analyze model performance across different training iterations, comparing our method with SiT and VFM (S-VFM without the straightness objective), as shown in Figure~\ref{fig:imagenet_train}. 
For SiT and VFM, we follow the default inference setting~\cite{ma2024sit, guo2025variational} with $\text{NFE}=250$, while for the proposed S-VFM, we use $\text{NFE}=10$ to reflect its few-step generation property. 
The results demonstrate a consistent performance improvement, further confirming the efficiency and effectiveness of S-VFM in both training and inference.

\subsection{Ablation Study}
\begin{figure}[htbp]
  \centering
  \includegraphics[width=0.95\linewidth]{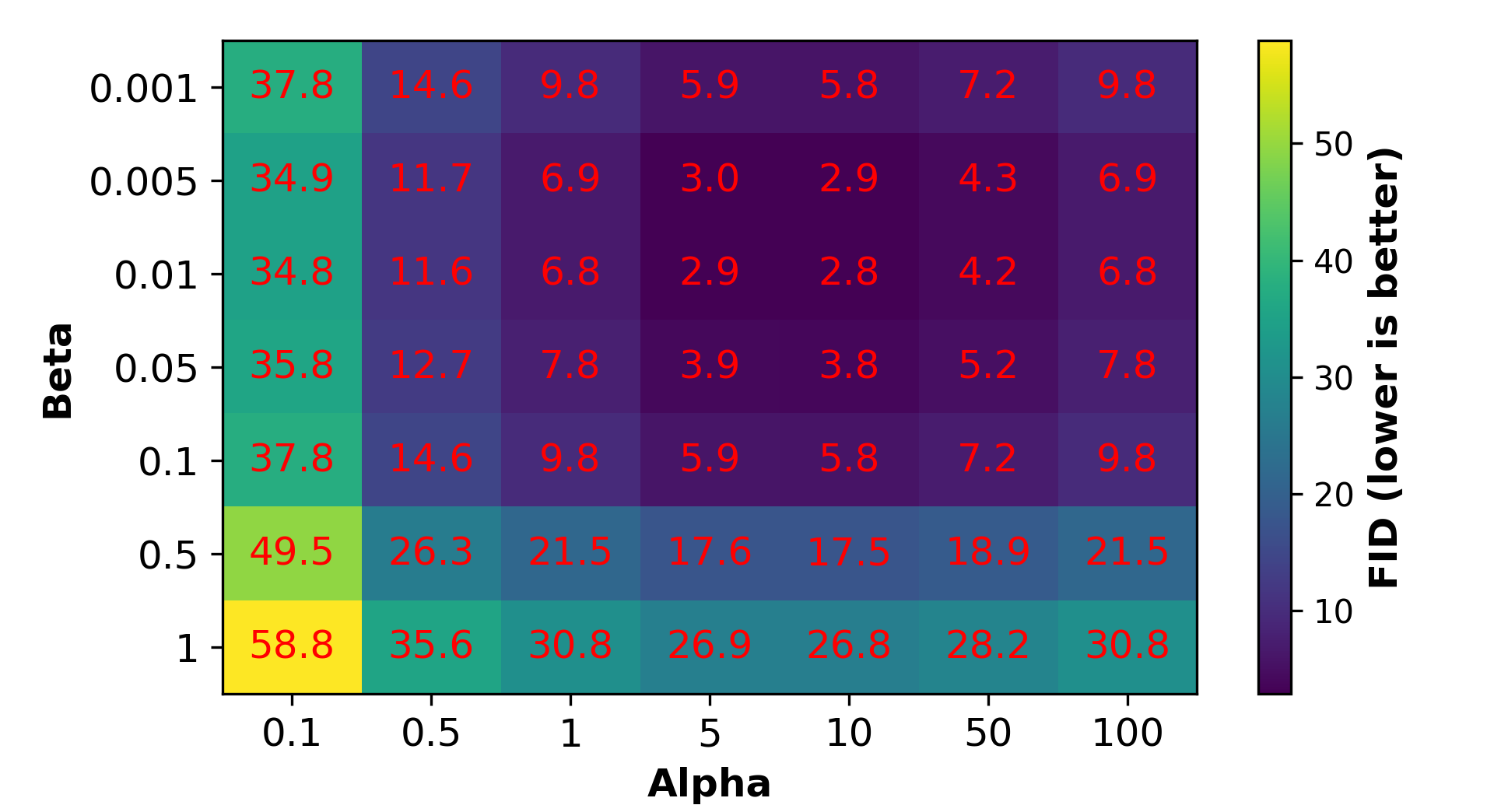}
    
  \vspace{-1em}
\caption{
\textbf{Ablation Study Results by Adjusting $\alpha$ and $\beta$ on CIFAR-10 Dataset.} The value in the grid represents the FID score of one-step generation (NFE = $1$), lower is better.
}

  \label{fig:ablation}
\end{figure}


To evaluate the contribution of each hyperparameter to S-VFM, we conduct an ablation study by evaluating the generation performance while varying the hyperparameters $\alpha$ and $\beta$. 
Figure~\ref{fig:ablation} presents the one-step generation FID scores for different combinations of $\alpha$ and $\beta$ on the CIFAR-10 dataset in a heatmap format, where lower values indicate better performance. 
The results show that S-VFM exhibits a stable performance plateau around the optimal region, suggesting the method is relatively insensitive to hyperparameter variations and demonstrating its adaptive robustness.

\section{Conclusion}
\label{sec:conclusion}

In this work, we address the key contradiction between straight trajectory learning and the independent coupling implemented in Flow Matching by introducing the Straight Variational Flow Matching framework. By incorporating a variational latent code that offers a global ``generation overview,'' allowing the model to distinguish between intersecting interpolants. Additionally, by using a straightness objective that penalizes time variations in the velocity field, the model learns to generate naturally straight trajectories.

\section*{Acknowledgment}

This manuscript was co-authored by Oak Ridge National Laboratory (ORNL), operated by UT-Battelle, LLC under Contract No. DE-AC05-00OR22725 with the U.S. Department of Energy. Any subjective views or opinions expressed in this paper do not necessarily represent those of the U.S. Department of Energy or the United States Government.

{
    \small
    \bibliographystyle{ieeenat_fullname}
    \bibliography{main}
}


\clearpage
\setcounter{page}{1}
\maketitlesupplementary

\section{Proofs for Straight, Non-Intersecting Interpolations}
\label{app_sec:proof}

\noindent
This appendix develops, in a self-contained way, the analytical objects and proofs
underlying our method. We introduce a \emph{straight, non-intersecting interpolation} \(Z\)
that is \emph{compatible} with the linear interpolation induced by a given coupling \(X\),
and we establish: (i) preservation of time marginals, (ii) reduction of convex transport costs,
and (iii) equivalent characterizations of straight (non-intersecting) couplings.

\bigskip
\noindent\textbf{Notation.}
For a random process \(X=\{X_t\}_{t\in[0,1]}\), write \(\mathrm{Law}(X_t)\) for its marginal law at time \(t\),
and \(\mathbb{E}[\cdot]\) for expectation.
For a coupling \((X_0,X_1)\) we denote \(\Delta^X:=X_1-X_0\) and the linear interpolant
\(X_t=(1-t)X_0 + t X_1\).
Conditional expectations such as \(\mathbb{E}[\Delta^X\mid X_t=x]\) are understood whenever they exist.

\subsection{Preliminaries}

\begin{definition}[Coupling, linear interpolation, and conditional velocity]
Let \((X_0,X_1)\) be any coupling on \(\mathbb{R}^d\) with joint density
\(\rho(x_0,x_1)=\rho_0(x_0)\,\rho_1(x_1)\).
Define the linear interpolation
\begin{align*}
X_t=(1-t)X_0+tX_1, \qquad t\in[0,1]\\ \Delta^X=v^X(X_t,t\, \big| \,(X_0,X_1))=X_1-X_0.
\end{align*}
The \emph{marginal velocity} associated with \(X\) is
\begin{align*}
v^X(x,t) \;&=\; \mathbb{E}\big[\Delta^X \,\big|\, X_t=x\big],\qquad (x,t)\in \mathbb{R}^d\times[0,1] \\
& = \int \Delta^X \cdot p(X_0,X_1 \, \big| \, X_t=x) \,d (X_0,X_1).
\end{align*}
When \(x\notin \mathrm{supp}(X_t)\), we set \(v^X(x,t)=0\).
All statements below compare \(v^X\) only on sets where it is evaluated against a marginal law: \(\mathrm{Law}(X_t)\).
\end{definition}

\begin{definition}[Rectifiability of \(X\)] \label{app:def:rectifiability}
We say that \(X\) is \emph{rectifiable} if \(v^X(\cdot,t)\) is locally bounded for each \(t\) and the continuity equation
\[
\partial_t \pi_t + \nabla\!\cdot\!\big(v^X(\cdot,t)\,\pi_t\big) \;=\;0,\qquad \pi_{t=0}=\mathrm{Law}(X_0),
\]
admits a unique solution \(\{\pi_t\}_{t\in[0,1]}\).
Equivalently, the ordinary differential equation \(\dot X_t=v^X(X_t,t)\) admits a unique flow of characteristics.
\end{definition}

\begin{definition}[Non-intersection functional]
For any coupling \((X_0,X_1)\) with linear interpolant \(X_t\), define
\[
V\big((X_0,X_1)\big) \;:=\; \int_0^1 \mathbb{E}\!\left[\big\|\,\Delta^X - \mathbb{E}[\Delta^X \mid X_t]\,\big\|^2\right]\,dt.
\]
\end{definition}

\begin{lemma}[Non-intersection if and only if zero conditional variance]
\label{app:lem:non_intersection_equiv}
For a coupling \((X_0,X_1)\) with linear interpolation \(X_t\), the following are equivalent:
\begin{enumerate}
\item For two independent identically distributed couplings \((X_0,X_1)\) and \((X_0',X_1')\),
\begin{align*}
\exists t\in(0,1): &(1-t)X_0+tX_1=(1-t)X_0'+tX_1' \\
&\mathbb{P}\!\left[(X_0,X_1)\neq(X_0',X_1')\right]=0.
\end{align*}
\item \(V\big((X_0,X_1)\big)=0\); equivalently \(\Delta^X=\mathbb{E}[\Delta^X\mid X_t]\) for \(t\in[0,1]\).
\end{enumerate}
\end{lemma}
\begin{proof}
\(\,(1)\Rightarrow(2)\,\): Non-intersection implies that the slope \(\Delta^X\) is a measurable function of \((X_t,t)\), hence \(\mathrm{Var}(\Delta^X\mid X_t)=0\);
integrating over \(t\) gives \(V=0\).
\(\,(2)\Rightarrow(1)\,\): If \(\Delta^X=\mathbb{E}[\Delta^X\mid X_t]\), then the slope through any \(X_t\) is unique, two distinct lines cannot share a point at the same time unless they coincide.
\end{proof}

\subsection{The straight interpolation \texorpdfstring{$Z$}{Z} compatible with \texorpdfstring{$X$}{X}}

\begin{definition}[Straight interpolation compatible with \(X\)]
\label{app:def:Z}
A process \(Z=\{Z_t\}_{t\in[0,1]}\) on the same probability space as \(X\) is called a
\emph{straight interpolants} compatible with Flow Matching trajectories \(X\) if the following hold:
\begin{itemize} [leftmargin=2em]
\item[(Z1)] \textbf{Linear paths.} There exist random endpoints \((Z_0,Z_1)\) with \(Z_t=(1-t)Z_0+tZ_1\) and \(\Delta^Z:=Z_1-Z_0\).
\item[(Z2)] \textbf{Non-intersection.} \(V\big((Z_0,Z_1)\big)=0\). Equivalently, \(\Delta^Z=\mathbb{E}[\Delta^Z\mid Z_t]\) for \(t\in[0,1]\).
\item[(Z3)] \textbf{Velocity-field matching.} With \(\mu_t=\mathrm{Law}(Z_t)\), for \(x \sim \mu_t\) and \(t\in[0,1]\).
\[
v^Z(x,t):=\mathbb{E}[\Delta^Z\mid Z_t=x]=v^X(x,t).
\]
\item[(Z4)] \textbf{Initialization.} \(Z_0=X_0\).
\end{itemize}
\end{definition}

\begin{remark}[Immediate consequences]
\label{app:rem:conseq}
From (Z2), \(v^Z(Z_t,t)=\Delta^Z\), hence
\[
Z_1-Z_0=\int_0^1 v^Z(Z_t,t)\,dt.
\]
By (Z3), \(v^Z=v^X\) at \(x \sim \mu_t\), so we also have
\[
\Delta^Z = \int_0^1 v^X(Z_t,t)\,dt.
\]
\end{remark}

\subsection{Main results}

\begin{theorem}[Marginal preservation]
\label{app:thm:marginal_preservation}
Assume \(X\) is rectifiable and \(Z\) satisfies Definition~\ref{app:def:Z}.
Then \(\mathrm{Law}(Z_t)=\mathrm{Law}(X_t)\) for all \(t\in[0,1]\).
\end{theorem}

\begin{proof}
Let \(h: \mathbb{R}^d \rightarrow \mathbb{R}\) be any compactly supported continuously differentiable test function. Using (Z1) and (Z2),
\begin{align*}
\frac{d}{dt}\mathbb{E}[h(Z_t)]
=\mathbb{E}[\nabla h(Z_t)^\top \dot Z_t]
&=\mathbb{E}[\nabla h(Z_t)^\top \Delta^Z] \\
&=\mathbb{E}[\nabla h(Z_t)^\top v^Z(Z_t,t)].
\end{align*}
By (Z3), \(v^Z(\cdot,t)=v^X(\cdot,t)\) for \(\mu_t\), therefore
\[
\frac{d}{dt}\mathbb{E}[h(Z_t)] = \mathbb{E}[\nabla h(Z_t)^\top v^X(Z_t,t)].
\]
This shows that \(\mu_t=\mathrm{Law}(Z_t)\) solves the continuity equation
\(\partial_t \mu_t+\nabla\!\cdot\!\big(v^X(\cdot,t)\mu_t\big)=0\) with
\(\mu_0=\mathrm{Law}(Z_0)=\mathrm{Law}(X_0)\) by (Z4).
By rectifiability of \(X\), the solution is unique and equals \(\pi_t=\mathrm{Law}(X_t)\).
Hence \(\mu_t=\pi_t\) for all \(t\in[0,1]\).
\end{proof}

\begin{theorem}[Convex transport-cost reduction]
\label{app:thm:convex_cost}
Under the assumptions of Theorem~\ref{app:thm:marginal_preservation}, for any convex \(c:\mathbb{R}^d\to\mathbb{R}\),
\[
\mathbb{E}\!\left[c(Z_1-Z_0)\right] \;\le\; \mathbb{E}\!\left[c(X_1-X_0)\right].
\]
If \(c\) is strictly convex, equality holds if and only if \(V\big((X_0,X_1)\big)=0\).
\end{theorem}

\begin{proof}
From (Z2),
\[
Z_1-Z_0=\int_0^1 v^Z(Z_t,t)\,dt.
\]
By Jensen’s inequality in time,
\[
\mathbb{E}\!\left[c(Z_1-Z_0)\right]
\le \int_0^1 \mathbb{E}\!\left[c\big(v^Z(Z_t,t)\big)\right]\,dt.
\]
Using (Z3) and Theorem~\ref{app:thm:marginal_preservation},
\begin{align*}
\int_0^1 \mathbb{E}\!\left[c\big(v^Z(Z_t,t)\big)\right]dt
&= \int_0^1 \mathbb{E}\!\left[c\big(v^X(X_t,t)\big)\right]dt \\
&= \int_0^1 \mathbb{E}\!\left[c\big(\mathbb{E}[\Delta^X\mid X_t]\big)\right]dt.
\end{align*}
A conditional Jensen inequality (conditioning on \(X_t\)) yields
\[
\mathbb{E}\!\left[c\big(\mathbb{E}[\Delta^X\mid X_t]\big)\right]
\le \mathbb{E}\!\left[\mathbb{E}\!\left[c(\Delta^X)\mid X_t\right]\right]
= \mathbb{E}\!\left[c(\Delta^X)\right].
\]
Combining gives the inequality. If \(c\) is strictly convex, equality forces tightness of the conditional Jensen step, hence \(\Delta^X=\mathbb{E}[\Delta^X\mid X_t]\), which by Lemma~\ref{app:lem:non_intersection_equiv} is \(V((X_0,X_1))=0\).
\end{proof}

\begin{theorem}[Equivalent characterizations of straight interpolants]
\label{app:thm:equivalences}
Assume \(X\) is rectifiable and \(Z\) satisfies Definition~\ref{app:def:Z}.
The following are equivalent: when they hold, we say $X$ yielded from couplings \((X_0,X_1)\) are \emph{straight interpolants}, which coincide with $Z$.
\begin{enumerate} [leftmargin=2em]
\item[(i)] \textbf{Tight convex cost.} There exists a strictly convex \(c\) with
\(\mathbb{E}[c(Z_1-Z_0)]=\mathbb{E}[c(X_1-X_0)]\).
\item[(ii)] \textbf{Endpoint coincidence.} \((Z_0,Z_1)=(X_0,X_1)\).
\item[(iii)] \textbf{Pathwise equality.} \(Z_t=X_t\) for all \(t\in[0,1]\).
\item[(iv)] \textbf{Non-intersection for \(X\).} \(V\big((X_0,X_1)\big)=0\).
\end{enumerate}
\end{theorem}

\begin{proof}
\begin{enumerate}
\item (iii) $\implies$ (ii) is immediate.
\item (ii) $\implies$ (i) turns the inequality of Theorem~\ref{app:thm:convex_cost} into equality.
\item (i) $\implies$ (iv): In Theorem~\ref{app:thm:convex_cost} with strictly convex \(c\), equality forces tightness of the conditional Jensen step; hence \(\Delta^X=\mathbb{E}[\Delta^X\mid X_t]\) for \(t\in[0,1]\), which is \(V((X_0,X_1))=0\) by Lemma~\ref{app:lem:non_intersection_equiv}.
\item(iv) $\implies$ (iii): If \(V((X_0,X_1))=0\), then for \(t\in[0,1]\), \(\dot X_t=\Delta^X=v^X(X_t,t)\).
For \(Z\), (Z2)–(Z3) give \(\dot Z_t=\Delta^Z=v^Z(Z_t,t)=v^X(Z_t,t)\).
Hence both \(X\) and \(Z\) solve \(\dot Y_t=v^X(Y_t,t)\) with the same initial value \(Y_0=X_0=Z_0\) by (Z4).
Uniqueness of rectifiability in Definition~\ref{app:def:rectifiability} yields \(Z\equiv X\).
\end{enumerate}
\end{proof}

\begin{corollary}[One-step generation]
\label{app:cor:one_step}
If any, and therefore all, of the items in Theorem~\ref{app:thm:equivalences} hold, then along each path
\(v^X(X_t,t)\equiv \Delta^X\) is constant in \(t\) and the ordinary differential equation \(\dot y_t=v^X(y_t,t)\) integrates in one step:
\[
X_1 \;=\; X_0 + \int_0^1 v^X(X_t,t)\,dt \;=\; X_0 + \Delta^X.
\]
\end{corollary}

\subsection{Equivalence with vanishing time derivative}

\begin{definition}[Time derivative]
For a differentiable vector field \(v:\mathbb{R}^d\times[0,1]\to\mathbb{R}^d\), the time derivative along its characteristics is
\begin{align*}
\frac{d}{dt} v(x,t) &= D_t v(x,t)\ := \frac{\partial v}{\partial t} \frac{dt}{dt} \ +\ \frac{\partial v}{\partial x} \frac{dx}{dt}   \\
&= \frac{\partial}{\partial t}v(x,t)\ +\ \big(\nabla_x v(x,t)\big)\,v(x,t)
\end{align*}
\end{definition}

\begin{theorem}[Straightness if and only if vanishing time derivative along \(X\)]
\label{app:thm:material_derivative_equivalence}
Assume \(X\) is rectifiable and \(Z\) satisfies Definition~\ref{app:def:Z}.
Assume moreover that \(v^X\) is continuously differentiable in \((x,t)\). Then $V\big((X_0,X_1)\big)=0$, if and only if 
$D_t v^X(X_t,t)=0$, for \(t\in[0,1]\).
\end{theorem}

\begin{proof}
\(\Rightarrow\): If \(V=0\), Theorem~\ref{app:thm:equivalences} gives \(Z\equiv X\) and \(v^X(X_t,t)=\Delta^X\), which is constant in \(t\) along each path.
The chain rule implies
\(D_t v^X(X_t,t)=\partial_t v^X(X_t,t)+\nabla v^X(X_t,t)\,\dot X_t=0\).

\(\Leftarrow\): If \(D_t v^X(X_t,t)=0\) for \(t\in[0,1]\), then \(v^X(X_t,t)\) is constant in \(t\) along each path.
With \(Z_0=X_0\), the solution has the linear form \(X_t=(1-t)X_0+tX_1\) and distinct trajectories cannot intersect
because the ordinary differential equation with a well-posed vector field has unique solutions.
By Theorem~\ref{app:thm:marginal_preservation}, \(\mathrm{Law}(Z_t)=\mathrm{Law}(X_t)\) for all \(t\).
Applying Theorem~\ref{app:thm:equivalences} then yields \(V((X_0,X_1))=0\).
\end{proof}




\end{document}